\documentclass[letterpaper, 10 pt, journal, twoside]{IEEEtran}  

\IEEEoverridecommandlockouts                              

\usepackage{amsmath} 
\usepackage{amssymb}  
\usepackage{mathrsfs}
\usepackage{graphics} 
\usepackage{subfig}
\usepackage{epsfig} 

\newcommand{\Real}{\mathbb R}

\newcommand{\Integer}{\mathbb{Z}_{\geq0}}

\newtheorem{assumption}{\textbf{Assumption}}
\newtheorem{remark}{\textbf{Remark}}
\newtheorem{theorem}{\textbf{Theorem}}
\newtheorem{problem}{\textbf{Problem}}
\newtheorem{lemma}{\textbf{Lemma}}

\title{Quadrupedal Locomotion via Event-Based Predictive Control and QP-Based Virtual Constraints}

\author{Kaveh Akbari Hamed, Jeeseop Kim, and Abhishek Pandala
\thanks{*The work of K. Akbari Hamed is supported by the National Science Foundation (NSF) under Grant Numbers 1854898, 1906727, 1923216, and 1924617. The work of J. Kim and A. Pandala is supported by the NSF Grant Number 1854898. The content is solely the responsibility of the authors and does not necessarily represent the official views of the NSF.}
\thanks{$^{1}$K. Akbari Hamed, J. Kim, and A. Pandala are with the Department of Mechanical Engineering, Virginia Tech, Blacksburg, VA 24061 USA {\tt\small kavehakbarihamed@vt.edu}, {\tt\small jeeseop@vt.edu} and {\tt\small agp19@vt.edu}}%
}

\begin{document}

\maketitle


\begin{abstract}
This paper aims to develop a hierarchical nonlinear control algorithm, based on model predictive control (MPC), quadratic programming (QP), and virtual constraints, to generate and stabilize locomotion patterns in a real-time manner for dynamical models of quadrupedal robots. The higher level of the proposed control scheme is developed based on an event-based MPC that computes the optimal center of mass (COM) trajectories for a reduced-order linear inverted pendulum (LIP) model subject to the feasibility of the net ground reaction force (GRF). The asymptotic stability of a desired target point for the reduced-order model under the event-based MPC approach is investigated. It is shown that the event-based nature of the proposed MPC approach can significantly reduce the computational burden associated with the real-time implementation of MPC techniques. To bridge the gap between reduced- and full-order models, QP-based virtual constraint controllers are developed at the lower level of the proposed control scheme to impose the full-order dynamics to track the optimal trajectories while having all individual GRFs in the friction cone. The analytical results of the paper are numerically confirmed on full-order simulation models of a 22 degree of freedom quadrupedal robot, Vision 60, that is augmented by a robotic manipulator. The paper numerically investigates the robustness of the proposed control algorithm against different contact models.
\end{abstract}

\begin{IEEEkeywords}
Legged Robots, Motion Control, Multi-Contact Whole-Body Motion Planning and Control
\end{IEEEkeywords}


\vspace{-0.8em}
\section{INTRODUCTION}
\label{INTRODUCTION}
\vspace{-0.2em}

\IEEEPARstart{T}{he} overarching goal of this paper is to develop a hierarchical control algorithm, based on nonlinear control, model predictive control (MPC), and quadratic programming (QP), to generate and stabilize locomotion trajectories for complex dynamical models of quadrupedal robots in a real-time manner. The proposed approach employs a higher-level and event-based MPC at the beginning of each continuous-time domain (i.e., event) that generates optimal trajectories for a reduced-order linear inverted pendulum (LIP) model subject to the feasibility of the net ground reaction force (GRF). The stability of the system subject to event-based MPC is investigated to demonstrate that the MPC does not need to be solved at every time sample. This significantly reduces the computational burden associated with MPC-based path planning approaches of legged locomotion while guaranteeing stability. To reduce the difference between the reduced- and full-order models of locomotion, a QP-based nonlinear controller is solved at the lower level of the proposed approach to impose the full-order dynamics to track the optimal trajectories while keeping all individual GRFs feasible. It is shown that the developed control algorithm can address stable, robust, and aperiodic locomotion patterns in different directions with uncertainty in contact models.

\begin{figure}[t!]
\centering
\includegraphics[width=0.85\linewidth]{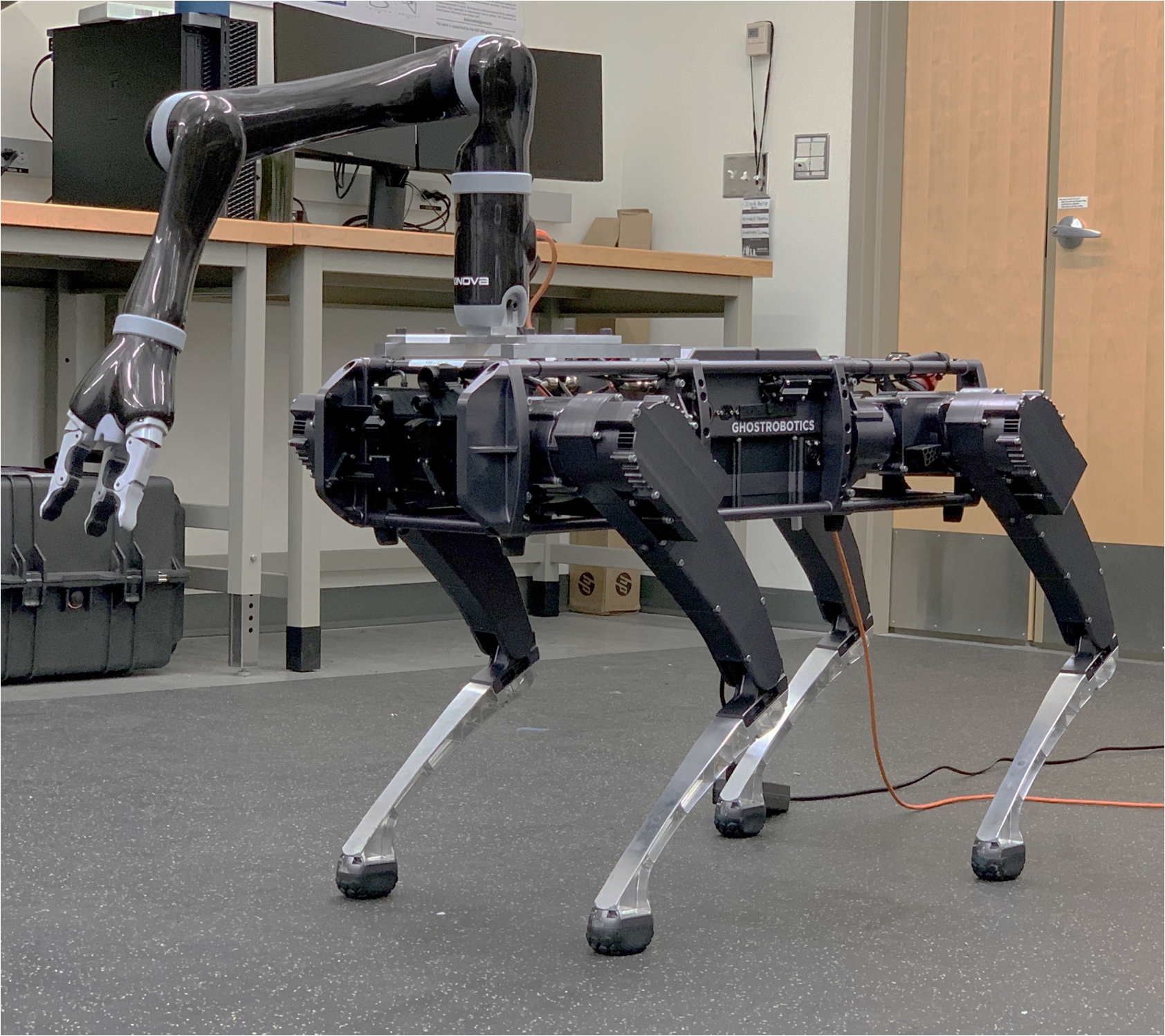}
\vspace{-0.3em}
\caption{Vision 60 robot, augmented with Kinova arm, whose full-order model will be used for the numerical simulations.}
\label{V60_Kinova}
\vspace{-1.5em}
\end{figure}

\vspace{-0.5em}
\subsection{Related Work and Motivation}

Hybrid systems theory has become a powerful approach for modeling and control of legged locomotion  \cite{Grizzle_Asymptotically_Stable_Walking_IEEE_TAC,Ames_RES_CLF_IEEE_TAC,Sreenath_Grizzle_HZD_Walking_IJRR,Veer_Poulakakis,Tedrake_Robus_Limit_Cycles_CDC,Byl_HZD,Johnson_Burden_Koditschek,Spong_Controlled_Symmetries_IEEE_TAC,Manchester_Tedrake_LQR_IJRR,Vasudevan2017,Hamed_Gregg_IEEE_TAC,Hamed_Buss_Grizzle_BMI_IJRR}. Existing nonlinear control techniques that address the hybrid nature of locomotion models have been developed based on hybrid reduction \cite{Ames_HybridReduction_Original_Paper}, controlled symmetries \cite{Spong_Controlled_Symmetries_IEEE_TAC}, transverse linearization \cite{Manchester_Tedrake_LQR_IJRR}, and hybrid zero dynamics (HZD) \cite{Westervelt_Grizzle_Koditschek_HZD_IEEE_TRO,Ames_RES_CLF_IEEE_TAC}.  In the HZD approach, a set of kinematic constraints, referred to as virtual constraints, is defined as outputs for the dynamical models of legged robots to coordinate the links of the robot within a stride. The virtual constraints are imposed by the action of a feedback controller (e.g., input-output (I-O) linearization \cite{Isidori_Book}). The virtual constraint controllers have been numerically and experimentally validated for the motion control of bipedal robots \cite{Cheavallereau_Grizzle_RABBIT,Sreenath_Grizzle_HZD_Walking_IJRR,Ramezani_Hurst_Hamed_Grizzle_ATRIAS_ASME,Ames_DURUS_TRO,Ames_RES_CLF_IEEE_TAC,Byl_HZD,Martin_Schmiedeler_IJRR}, powered prosthetic legs  \cite{Gregg_Toward_Biomimetic_Control_IEEE_CST,zhao2016multicontact} and exoskeletons \cite{agrawal2017first}. The gait planning in the HZD approach is typically formulated as a nonlinear programming (NLP) problem. Although the HZD-based optimization techniques effectively address the gait planning of planar (i.e., 2D) robots, they become computationally intensive for high degree of freedom (DOF) spatial (i.e., 3D) robots. Reference \cite{Ames_DURUS_TRO} developed a scalable gait planning approach based on HZD and direct collocation that can be effectively solved with existing NLP tools (see e.g., \cite{Ames_DURUS_TRO} for the bipedal robot DURUS and \cite{Hamed_Ma_Ames_Vision60} for the quadrupedal robot Vision 60). Although the direct-collocation based HZD approach generates optimal trajectories for full-order hybrid models of legged robots in a fast manner, it \textit{cannot} address real-time trajectory optimization in complex environments. This motivates the integration of convex optimization-based path planning techniques with the HZD framework.


MPC-based approaches integrated with reduced-order models have been used for real-time path planning of bipedal and quadrupedal locomotion, see e.g., \cite{Leonessa_Pratt_MPC,Ott_MPC,LIP_Original_Kajita,Pratt_LIP,little_dog_QP}. Most of these approaches address linear inverted pendulum (LIP) models for  bipedal locomotion while generating optimal trajectories for the center of mass (COM) and center of pressure (COP) of the robot subject to the zero moment point (ZMP)  conditions \cite{Vukobratovic_Book} and feasibility of the GRF. These techniques, however, cannot be easily extended to quadrupedal locomotion as the LIP-based MPC approaches do \textit{not} consider the feasibility of all individual GRFs on the contacting legs. To tackle this problem, \cite{Kim_Wensing_Convex_MPC_01,Park_Pandala_Ding_MPC_01,seminimpc} have developed an interesting convex optimization formulation based on MPC and centroidal dynamics. In particular, the MPC approach of \cite{Kim_Wensing_Convex_MPC_01,Park_Pandala_Ding_MPC_01,seminimpc} plans for the optimal GRFs of the contacting leg ends at every time sample (e.g., 200 Hz). The optimal GRFs are then employed by a lower-level controller to generate the required torques at the joint levels. The GRF planning-based approaches have been numerically and experimentally verified for agile quadrupedal locomotion. Although the MPC problems are typically formulated as convex QPs, they still have a significant amount of decision variables to be optimized at every time sample which makes the MPC-based techniques computationally intensive. On the other hand, the gap between reduced- and full-order dynamical models of increasingly sophisticated legged machines with heavy legs has not been addressed in the MPC-based techniques. In this paper, we aim to answer the following fundamental \textit{questions}: 1) how can we implement an MPC-based approach for real-time planning of legged locomotion at a slower rate, e.g., at the beginning of each continuous-time domain rather than every time sample, while guaranteeing the stability of the gaits, and 2) how  can we systematically bridge the gap between the reduced- and full-order dynamical models using the HZD approach?

\vspace{-0.5em}
\subsection{Objectives and Contributions}

The \textit{objectives} and \textit{key contributions} of this paper are as follows. The paper develops a hierarchical control algorithm based on an event-based MPC and QP-based virtual constraint controllers that generate and stabilize quadrupedal locomotion patterns in real time. In particular, we do not employ any NLP-based and computationally expensive HZD gait planning algorithm. The higher level of the proposed approach formulates a finite-time and QP-based MPC at the beginning of each continuous-time domain (i.e., event-based manner) to compute the optimal COM trajectories subject to the net GRF feasibility. This can address any locomotion pattern in different directions (e.g., forward, backward, sideways, in-place, and diagonal) via any predefined contact sequences with possible start and stop conditions. The stability of the target point for the LIP model with the proposed event-based MPC is addressed. It is shown that under some sufficient conditions, one would not need to solve the MPC problem at every-time sample to reach the target point. This significantly reduces the computational complexity of real-time MPC. To bridge the gap between the LIP and full-order dynamical model of locomotion, a low-level and QP-based virtual constraints controller is developed. The low-level controller imposes the full-order dynamics to track the optimal reduced-order trajectories while having all individual GRFs in the friction cone. Additionally, the lower-level QP has fewer decision variables compared to the event-based MPC (e.g., $50\%$). The analytical results of the paper are numerically confirmed on full-order simulation models of a 22 DOF quadrupedal robot, Vision 60, that is augmented by a Kinova robotic manipulator (see Fig. \ref{V60_Kinova}). The paper also investigates the robustness of the proposed control algorithm against different contact modeling approaches including the LuGre model \cite{LuGre_model} and the per-contact iteration technique \cite{RAISIM}. It is shown that the proposed controller can systematically generate and stabilize gaits in different directions with start and stop conditions. In \cite{NLIP_LIP_HyQ}, the authors presented an interesting LIP-based and nonlinear trajectory optimization framework to generate a wide range of quadrupedal gaits. The approach of the current paper completely differs from \cite{NLIP_LIP_HyQ} in that 1) we address event-based MPC for trajectory planning of the LIP model while addressing the asymptotic stability of the final target point, and 2) we reduce the gap between the reduced- and full-order dynamical models by setting up the QP-based virtual constraint controllers. Unlike the two-level control approach of \cite{kim2019highly}, the current paper 1) studies the asymptotic stability under the MPC approach, and then 2) extends the concept of HZD controllers, based on convex optimization, to quadrupedal locomotion.

\begin{figure*}[t!]
\centering
\vspace{0.1em}
\includegraphics[width=\linewidth]{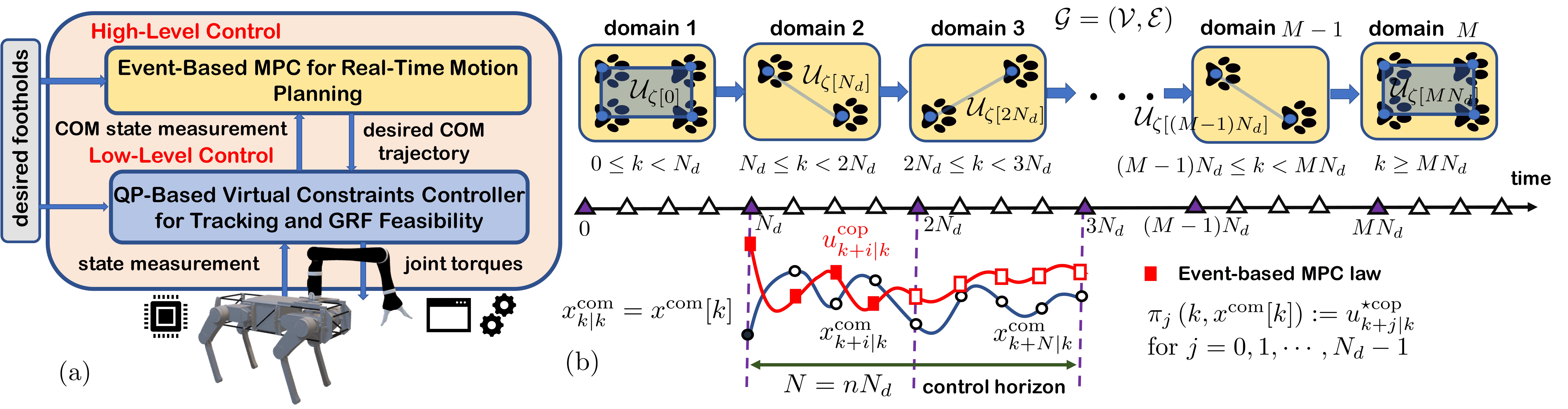}
\caption{(a) Illustration of the proposed hierarchical control algorithm, based on nonlinear control, QP, and event-based MPC. (b) Illustration of a locomotion pattern with the corresponding graph $\mathcal{G}=(\mathcal{V},\mathcal{E})$ and the event-based MPC law.}
\label{HierachicalControl_MPC_illustration}
\vspace{-1.5em}
\end{figure*}


\section{MODELS OF LEGGED LOCOMOTION}
\vspace{-0.2em}

We consider a full-order dynamical model of Vision 60 that is augmented by a Kinova robotic manipulator for locomotion and manipulation purpose. Vision 60 is a quadrupedal robot designed and manufactured by Ghost Robotics. The floating-base model of the composite robot consists of 22 DOFs of which 12 DOFs are actuated and assigned to legs. In particular, each leg of the robot has an actuated 2 DOF hip joint plus an actuated 1 DOF knee joint and ends at a point foot. In addition, 4 DOFs with 4 actuators are assigned to the Kionva manipulator. The remaining 6 DOFs are unactuated and describe the absolute position and orientation of the robot with respect to an inertial world frame. The generalized coordinates of the robot can be expressed as $q:=\textrm{col}(p_{b},\phi_{b},q_{\textrm{body}})\in\mathcal{Q}\subset\Real^{22}$, in which $p_{b}\in\Real^{3}$ and $\phi_{b}\in\Real^{3}$ describe the absolute position and orientation of the torso, respectively. Moreover, $q_{\textrm{body}}\in\Real^{16}$ represents the shape (i.e., internal joints) of the robot. The state vector of the mechanical system is taken as $x:=\textrm{col}(q,\dot{q})\in\mathcal{X}$, where $\mathcal{X}:=\textrm{T}\mathcal{Q}:=\mathcal{Q}\times\Real^{22}$ denotes the state space. In our notation, ``$\textrm{col}$'' represents the column vector. The control inputs (i.e., joint torques) are finally represented by $\tau\in\Real^{16}$.

The evolution of the robot can be expressed by the following ordinary differential equations (ODEs)
\begin{equation}\label{equation_of_motion}
D(q)\,\ddot{q}+H\left(q,\dot{q}\right)=\Upsilon\,\tau+\sum_{\ell\in\mathcal{C}}J_{\ell}^\top(q)\,F_{\ell},
\end{equation}
where $D(q)\in\Real^{22\times22}$ represents the positive definite mass-inertia matrix, $H(q,\dot{q})\in\Real^{22}$ denotes the Coriolis, centrifugal, and gravitational forces, and $\Upsilon\in\Real^{22\times16}$ represents the input distribution matrix. In our notation, $\mathcal{C}$ is the index set of contact points with the ground. Furthermore, for every $\ell\in\mathcal{C}$, $J_{\ell}(q)\in\Real^{3\times22}$ and $F_{\ell}\in\Real^{3}$ denote the corresponding contact Jacobian matrix and GRF, respectively. The contact forces can be computed using 1) the rigid contact assumption and hybrid system approach \cite{Jessy_Book,Hamed_Ma_Ames_Vision60}, 2) compliant contact models (e.g., LuGre model \cite{LuGre_model}), or 3) nonlinear and linear complementarity problems \cite{NCP_LCP} as well as optimization-based techniques \cite{MuJoCo,RAISIM}. We remark that the model \eqref{equation_of_motion} is valid if $F_{\ell}\in\mathcal{FC}$ for all $\ell\in\mathcal{C}$, where $\mathcal{FC}:=\{(F_{x},F_{y},F_{z})^\top|\,F_{z}>0,\,\pm{}F_{x}<\frac{\mu}{\sqrt2}F_{z},\,\pm{}F_{y}<\frac{\mu}{\sqrt{2}}F_{z}\}$ denotes the friction cone for some friction coefficient $\mu$. For later purposes, the equations of motion in \eqref{equation_of_motion} can be written in a state space form as
\begin{equation}\label{state_eq}
\dot{x}=f(x)+g(x)\,\tau+w(x)\,F,
\end{equation}
in which $F:=\textrm{col}\{F_{\ell}\,|\,\ell\in\mathcal{C}\}$ represents the contact forces.

\begin{remark}
We will assume a rigid contact model for the development of the lower-level nonlinear controller in Section \ref{QP-BASED VIRTUAL CONSTRAINTS FOR TRACKING AND FEASIBILITY}. The robustness of the proposed control scheme against different contact modeling approaches will be verified in Section \ref{NUMERICAL SIMULATIONS}.
\end{remark}


\section{EVENT-BASED PREDICTIVE CONTROL}
\label{EVENT-BASED PREDICTIVE CONTROL}
\vspace{-0.2em}

In this paper, we develop a two-level control algorithm that can steer a quadrupedal robot from an initial point to a final point in a stable and robust manner (see Fig. \ref{HierachicalControl_MPC_illustration}a). At the higher level of the control algorithm, we employ an event-based MPC for real-time motion planning of the COM. In particular, we formulate a finite-time optimal control problem based on MPC and convex QP to steer the state of a reduced-order LIP model subject to the feasibility of the net GRF. The lower-level controller is developed based on the concept of virtual constraints and QP. More specifically, a QP formulation is developed to solve the feedback linearization problem of the full-order model to track the optimal COM trajectory as well as desired footholds while satisfying the friction cone conditions for all individual GRFs. In this section, we address the event-based MPC approach. The QP-based virtual constraints controller will be presented in Section \ref{QP-BASED VIRTUAL CONSTRAINTS FOR TRACKING AND FEASIBILITY}.

\noindent\textbf{Reduced-Order LIP Model:} The LIP model can be described by the following ODEs \cite{LIP_Original_Kajita}
\begin{equation}\label{LIP}
\begin{bmatrix}
\ddot{r}^{\textrm{com}}_{x}\\
\ddot{r}^{\textrm{com}}_{y}
\end{bmatrix}=\frac{g}{r_{z}^{\textrm{com}}}\begin{bmatrix}
r^{\textrm{com}}_{x}-u^{\textrm{cop}}_{x}\\
r^{\textrm{com}}_{y}-u^{\textrm{cop}}_{y}
\end{bmatrix},
\end{equation}
where $r^{\textrm{com}}:=\textrm{col}(r^{\textrm{com}}_{x},r^{\textrm{com}}_{y})\in\Real^{2}$ denotes the Cartesian coordinates of the COM with respect to the inertial world frame, projected onto the $xy$-plane, $r_{z}^{\textrm{com}}$ represents the constant height of the COM, $g$ is the gravitational constant, and $u^{\textrm{cop}}:=\textrm{col}(u^{\textrm{cop}}_{x},u^{\textrm{cop}}_{y})\in\Real^{2}$ denotes the Cartesian coordinates of the COP. From \eqref{LIP}, the net GRF applied to the COM can be expressed as
\begin{equation}\label{net_force}
F_{\textrm{net}}:=\sum_{\ell\in\mathcal{C}}F_{\ell}=\!\!\begin{bmatrix}
m_{\textrm{tot}}\,\ddot{r}^{\textrm{com}}_{x}\\
m_{\textrm{tot}}\,\ddot{r}^{\textrm{com}}_{y}\\
m_{\textrm{tot}}g\\
\end{bmatrix}\!\!=m_{\textrm{tot}}g\begin{bmatrix}
\frac{1}{r_{z}^{\textrm{com}}}\left(r^{\textrm{com}}_{x}-u^{\textrm{cop}}_{x}\right)\\
\frac{1}{r_{z}^{\textrm{com}}}\left(r^{\textrm{com}}_{y}-u^{\textrm{cop}}_{y}\right)\\
1
\end{bmatrix},
\end{equation}
in which $m_{\textrm{tot}}$ represents the total mass of the robot. By defining the LIP state vector $x^{\textrm{com}}:=\textrm{col}(r^{\textrm{com}}_{x},\dot{r}^{\textrm{com}}_{x},r^{\textrm{com}}_{y},\dot{r}^{\textrm{com}}_{y})\in\Real^{4}$ and employing the zero-order hold (ZOH) discretization approach for some sampling time $T_{d}$, the ODEs in \eqref{LIP} can be discretized as follows
\begin{equation}\label{LIp_discrete}
    x^{\textrm{com}}[k+1]=A_{d}\,x^{\textrm{com}}[k] + B_{d}\,u^{\textrm{cop}}[k],
\end{equation}
where $k\in\Integer$ represents a non-negative integer with $A_{d}\in\Real^{4\times4}$ and $B_{d}\in\Real^{4\times2}$ being the state and input matrices, respectively.

\noindent\textbf{Steering Problem:} We are interested in steering the discrete-time dynamics \eqref{LIp_discrete} from an initial state to a final state over $M$ continuous-time domains for some positive integer $M\geq1$. We consider a general locomotion pattern with an arbitrary sequence of double-, triple-, or quadruple-contact domains. Unlike our previous work \cite{Hamed_Ma_Ames_Vision60}, the locomotion pattern does \textit{not} need to be cyclic. In particular, we address \textit{aperiodic} locomotion with start and stop options. To make this notion more precise, we consider a directed graph $\mathcal{G}=(\mathcal{V},\mathcal{E})$ for the desired locomotion pattern, whose \textit{vertex set} represents the ordered sequence of continuous-time domains. In addition, $\mathcal{E}\subset\mathcal{V}\times\mathcal{V}$ denotes the \textit{edge set} to represent transitions (see Fig. \ref{HierachicalControl_MPC_illustration}b). Let us suppose that each continuous-time domain consists of $N_{d}\geq1$ time samples (i.e., grid points). We then define the \textit{domain indicator function} as $\zeta:\Integer\rightarrow\{1,2,\cdots,M\}$ by $\zeta[k]:=\lfloor\frac{k}{N_{d}}\rfloor+1$ for $0\leq{k}<MN_{d}$ and $\zeta[k]:=M$ for $k\geq{M}$ to assign the domain index for every time sample $k\in\Integer$. Here, $\lfloor{.}\rfloor$ represents the floor function. For the feasibility of the LIP model, we assume that the control input $u^{\textrm{cop}}[k]$ lies in the support polygon which is defined as the convex hull of the contacting points with the ground. That is,
$u^{\textrm{cop}}[k]\in\mathcal{U}_{\zeta[k]}$ for all $k\in\Integer$, where $\mathcal{U}_{\zeta[k]}\subset\Real^{2}$ is the corresponding convex hull for the domain $\zeta[k]$. If we define the \textit{contact coordinates matrix} for the domain $\zeta[k]$ as $C_{\zeta[k]}$ whose columns represent the Cartesian coordinates of the contacting feet with the ground, $u^{\textrm{cop}}[k]\in\mathcal{U}_{\zeta[k]}$ is equivalent to the existence of a time-varying vector $\lambda[k]$ such that
\begin{equation}\label{Feasibity_COP_02}
\textbf{0}\leq\lambda[k]\leq\textbf{1},\,\,\, \textbf{1}^\top\lambda[k]=1,\,\,\, u^{\textrm{cop}}[k]=C_{\zeta[k]}\lambda[k].
\end{equation}
We remark that in our notation, $\textbf{0}$ and $\textbf{1}$ denote vectors whose elements are zero and one, respectively. In addition, for the feasibility of the LIP  model, the net force must lie in the friction cone, i.e., $F_{\textrm{net}}\in\mathcal{FC}$. This latter condition together with \eqref{net_force} can be expressed as
\begin{equation}\label{friction_cone_net}
\Phi\,x^{\textrm{com}}[k]+\Psi\,u^{\textrm{cop}}[k]\leq\eta,\quad\forall k\in\Integer
\end{equation}
for some proper $\Phi$ and $\Psi$ matrices and some proper $\eta$ vector.

\begin{problem}[Optimal Steering Problem]\label{optimal_control_problem}
For a given locomotion graph $\mathcal{G}=(\mathcal{V},\mathcal{E})$, a phase index function $\zeta:\Integer\rightarrow\{1,2,\cdots,M\}$, a set of known contact coordinates matrices $\{C_{\zeta[k]}\}_{k\in\Integer}$, an initial state $x^{\textrm{com}}_{0}$, a final state $x^{\textrm{com}}_{f}$, and a steering time $T_{f}\in\Integer$, the optimal steering problem consists of finding an optimal sequence of control (i.e., COP) inputs $u^{\textrm{cop}}[k]$ for $0\leq{k}\leq{T_{f}-1}$ that steer \eqref{LIp_discrete} from $x^{\textrm{com}}_{0}$ to $x^{\textrm{com}}_{f}$ subject to \eqref{Feasibity_COP_02} and \eqref{friction_cone_net}.
\end{problem}

\noindent\textbf{Event-Based MPC:} To address Problem \ref{optimal_control_problem}, we set up an event-based MPC that is solved at the beginning of each domain (i.e., event) with some control horizon $N=nN_{d}$ and $n\geq1$. In particular, for every time sample $k=mN_{d}$ with $m\in\Integer$ (i.e., beginning of each continuous-time domain), we consider the following finite-time optimal control problem
\begin{alignat}{8}
& &&\!\!\min_{U^{\textrm{cop}}_{k\rightarrow{}k+N-1|k}}\!\!\!\! &&\mathcal{J}_{k}\left(x^{\textrm{com}}[k],U^{\textrm{cop}}_{k\rightarrow k+N-1|k}\right)
\!&&=p\left(x^{\textrm{com}}_{k+N|k}\right)\nonumber\\
& && && &&+\!\!\sum_{i=0}^{N-1}\!\!\mathcal{L}\left(x^{\textrm{com}}_{k+i|k},u^{\textrm{cop}}_{k+i|k}\right)\nonumber\\
& && \quad\,\,\textrm{s.t.} && \qquad x^{\textrm{com}}_{k+i+1|k}=A_{d}\,x^{\textrm{com}}_{k+i|k}&&+B_{d}\,u^{\textrm{cop}}_{k+i|k}\nonumber\\
& && && \qquad \Phi\,x_{k+i|k}^{\textrm{com}}+\Psi\,u_{k+i|k}^{\textrm{cop}}&&\leq\eta\nonumber\\
& && && \qquad u_{k+i|i}^{\textrm{cop}}\in\mathcal{U}_{\zeta[k+i]},\,\, &&i=0,1,\cdots,N-1\nonumber\\
& && && \qquad x_{k|k}^{\textrm{com}}=x^{\textrm{com}}[k], && \label{realtime_MPC}
\end{alignat}
where $U_{k\rightarrow{}k+N-1|k}^{\textrm{cop}}:=\textrm{col}(u^{\textrm{cop}}_{k|k},\cdots,u^{\textrm{cop}}_{k+N-1|k})$ and $x^{\textrm{com}}_{k+i|k}$ represents the estimated state vector at time $k+i$ predicted at time $k$ according to the recursive law $x^{\textrm{com}}_{k+i+1|k}=A_{d}\,x^{\textrm{com}}_{k+i|k}+B_{d}\,u^{\textrm{cop}}_{k+i|k}$ starting from the current state $x^{\textrm{com}}_{k|k}:=x^{\textrm{com}}[k]$. In an analogous manner, $u^{\textrm{cop}}_{k+i|k}$ denotes the COP input at time $k+i$ computed at time $k$. Furthermore, $p(x^{\textrm{com}}_{k+N|k})$ and $\mathcal{L}(x_{k+i|k}^{\textrm{com}},u_{k+i|k}^{\textrm{cop}})$ are the terminal and stage costs, respectively, defined as $p(x^{\textrm{com}}_{k+N|k}):=\|x^{\textrm{com}}_{k+N|k}\ - d^{\textrm{com}}_{k+N|k}\|_{P}^{2}$ and $\mathcal{L}(x_{k+i|k}^{\textrm{com}},u_{k+i|k}^{\textrm{cop}}):=\|x_{k+i|k}^{\textrm{com}}-d^{\textrm{com}}_{k+i|k}\|_{Q}^{2}+\|u_{k+i|k}^{\textrm{cop}}\|_{R}^{2}$
for some positive definite matrices $P\in\Real^{4\times4}$, $Q\in\Real^{4\times4}$, and $R\in\Real^{2\times2}$, in which $\|z\|_{P}^{2}:=z^\top{}Pz$. In our notation, $d_{k+i|k}^{\textrm{com}}$ represents a desired state trajectory for $x^{\textrm{com}}_{k+i|k}$ that is smooth in $i$ (e.g., linear) while starting at the current state $x^{\textrm{com}}[k]$ and ending at the final state $x^{\textrm{com}}_{f}$. Let $U^{\star}_{k\rightarrow{}k+N-1|k}:=\textrm{col}(u^{\star\textrm{cop}}_{k|k},\cdots,u^{\star\textrm{cop}}_{k+N-1|k})$ be the optimal solution of the problem \eqref{realtime_MPC}. Then in our proposed approach, the first $N_{d}$ components of $U^{\star}_{k\rightarrow{}k+N-1|k}$, that correspond to the time samples of the current continuous-time domain, are employed to the system \eqref{LIp_discrete}, that is,
\begin{equation}\label{MPC_law}
    u^{\textrm{cop}}[k+j] = u^{\star\textrm{cop}}_{k+j|k},\quad j=0,1,\cdots,N_{d}-1.
\end{equation}
For later purposes, we assume that $\pi_{j}(k,x^{\textrm{com}}[k]):=u^{\star\textrm{cop}}_{k+j|k}$ for $j=0,1,\cdots,N_{d}-1$ represents the MPC law computed at time $k$ (see Fig. \ref{HierachicalControl_MPC_illustration}b). The evolution of the closed-loop LIP model can then be expressed as
\begin{alignat}{4}
& x^{\textrm{com}}[mN_{d}+j] &&= A_{d}^{j}\,x^{\textrm{com}}[mN_{d}] \nonumber\\
& &&+ \sum_{\ell=0}^{j-1} A_{d}^{j-1-\ell}B_{d}\,\pi_{\ell}\left(mN_{d},x^{\textrm{com}}[mN_{d}]\right),\label{closed_loop_LIP}
\end{alignat}
for every $m\in\Integer$ and $j=0,1,\cdots,N_{d}$. In addition, we can define the following \textit{down-sample closed-loop system}
\begin{alignat}{4}
& x^{\textrm{com}}[(m+1)N_{d}] &&= A_{d}^{N_{d}}\,x^{\textrm{com}}[mN_{d}] \nonumber\\
& &&+ \sum_{\ell=0}^{N_{d}-1} A_{d}^{N_{d}-1-\ell}B_{d}\,\pi_{\ell}\left(mN_{d},x^{\textrm{com}}[mN_{d}]\right)\nonumber\\
& &&=: \Delta_{\textrm{down}}\left(mN_{d},x^{\textrm{com}}[mN_{d}]\right)\label{downsample_closed_loop_LIP}
\end{alignat}
for all $m\in\Integer$ whose state is updated every $N_{d}$ samples (i.e., based on events). The down-sample closed loop system will be studied for the stability analysis of the proposed control approach in Section \ref{STABILITY AND CONVERGENCE ANALYSIS}.

\begin{remark}
We remark that the MPC problem is solved at the beginning of each continuous-time domain (i.e., event) which corresponds to $k=mN_{d}$ and $m<M$. For all $m\geq M$, we still continue solving the MPC problem for every $k=mN_{d}$. It is worth mentioning that for these time values, $\zeta[k+i]=M$, and hence, $\mathcal{U}_{\zeta[k+i]}=\mathcal{U}_{M}$ which makes the optimal control problem \eqref{realtime_MPC} well-defined. During these time samples, the robot does not take any further steps (i.e., last domain) and only moves its COM to reach the final state.
\end{remark}

\begin{remark}\label{PD_QP}
We remark that the MPC formulation \eqref{realtime_MPC} together with \eqref{Feasibity_COP_02} can be expressed as QP in terms of the decisions variables $\{x^{\textrm{com}}_{k+i|k}\}_{i=1}^{N}$, $\{u^\textrm{cop}_{k+i|k}\}_{i=0}^{N-1}$, and $\{\lambda_{k+i|k}\}_{i=0}^{N-1}$ to retain the sparsity structure of \cite{Boyd_FastMPC}. To make the cost function of this QP positive definite in terms of all decision variables, one can add a term corresponding to $\lambda_{k+i|k}$, i.e.,
\begin{equation*}
\mathcal{J}_{k}\!=p\left(x^{\textrm{com}}_{k+N|k}\right)\!+\sum_{i=0}^{N-1}\left\{\mathcal{L}\left(x^{\textrm{com}}_{k+i|k},u^{\textrm{cop}}_{k+i|k}\right)+\mathcal{H}\left(\lambda_{k+i|k}\right)\right\},
\end{equation*}
where $\mathcal{H}(\lambda_{k+i|k}):=\|\lambda_{k+i|k}-\lambda^{\textrm{des}}_{k+i|k}\|_{\hat{R}}^{2}$ for some desired trajectory $\lambda_{k+i|k}^{\textrm{des}}$ and some positive definite matrix $\hat{R}$.
\end{remark}


\section{ASYMPTOTIC STABILITY ANALYSIS}
\label{STABILITY AND CONVERGENCE ANALYSIS}
\vspace{-0.2em}

The objective of this section is to address the asymptotic stability property of the target state for the closed-loop system. We aim to establish a connection between the asymptotic stability of the target state for the down-sample closed-loop system and that of the the original closed-loop system. Without loss of generality, we assume that the target state is taken at the origin, i.e., $x^{\textrm{com}}_{f}=0$. We then make the following assumption.

\begin{assumption}\label{Lipshitz_continuity}
The MPC problem is formulated as QP with a positive definite cost function as mentioned in Remark \ref{PD_QP}. In addition, the MPC is feasible for every $k=mN_{d}$ with optimal laws satisfying the conditions $\pi_{j}(mN_{d},0)=0$ for all $m\in\Integer$ and $j=0,1,\cdots,N_{d}-1$.
\end{assumption}

\begin{lemma}\label{Lipshitz_continuity_lemma}\textit{\textbf{(Lipschitz Continuity):}}
\textit{Suppose that Assumption \ref{Lipshitz_continuity} is satisfied. Then, the MPC laws are locally Lipschitz in $x^{\textrm{com}}$, i.e., there exists $\rho_{m,j}>0$ for all $m\in\Integer$ and $j=0,1,\cdots,N_{d}-1$ such that $\|\pi_{j}(mN_{d},x^{\textrm{com}})\|\leq\rho_{m,j}\|x^{\textrm{com}}\|$ for all $x^{\textrm{com}}$ in an open neighborhood of the origin.}
\end{lemma}

\begin{proof}
We can show that the QP problem of Remark \ref{PD_QP} can be expressed as a canonical QP (CQP) \cite{Coroianu2016}, that is,  $\min_{\xi}\{\frac{1}{2}\xi^\top\mathcal{P}\xi+\psi(x^{\textrm{com}})^\top\xi|\,\mathcal{A}\,\xi\geq\phi(x^{\textrm{com}}),\xi\geq0\}$ with some positive definite matrix $\mathcal{P}$, a matrix $\mathcal{A}$, and some vectors $\psi(x^{\textrm{com}})$ and $\phi(x^{\textrm{com}})$ that depend smoothly on the current state $x^{\textrm{com}}$ of the LIP model. Applying \cite[Theorem 2.1]{Coroianu2016} together with $\pi_{j}(mN_{d},0)=0$ completes the proof.
\end{proof}

Now we are in a position to present the following theorem.

\begin{theorem} \label{stability_theorem} \textit{\textbf{(Stability Analysis based on the Down-Sample System):}} \textit{Consider the original and down-sample discrete-time systems \eqref{closed_loop_LIP} and \eqref{downsample_closed_loop_LIP} with the MPC laws satisfying Assumption \ref{Lipshitz_continuity}. Suppose further that the Lipschitz constants $\rho_{m,j}$ are bounded for all $m\in\Integer$ and all $j=0,1,\cdots,N_{d}-1$, that is, $\rho_{m,j}\leq\rho$ for some $\rho>0$.
If the origin is uniformly asymptotically stable for the down-sample discrete-time system \eqref{downsample_closed_loop_LIP}, then it is asymptotically stable for the original system \eqref{closed_loop_LIP}.}
\end{theorem}

\begin{proof}
Since the origin is uniformly asymptotically stable for the down-sample system \eqref{downsample_closed_loop_LIP}, there is a class $\mathcal{KL}$ function $\beta$ such that
\begin{equation}\label{class_KL}
\left\|x^{\textrm{com}}[mN_{d}]\right\|\leq\beta\left(\left\|x^{\textrm{com}}[0]\right\|,mN_{d}\|\right)
\end{equation}
for all $m\in\Integer$ and every initial condition in an open neighborhood of the origin. From \eqref{closed_loop_LIP}, norm properties, Lemma \ref{Lipshitz_continuity_lemma}, and \eqref{class_KL}, one can conclude that
\begin{alignat}{4}
&\left\|x^{\textrm{com}}[mN_{d}+j]\right\|&&\leq\|A_{d}^{j}\|\left\|x^{\textrm{com}}[mN_{d}]\right\|\nonumber\\
& &&+\sum_{\ell=0}^{j-1}\|A_{d}^{j-1-\ell}\|\|B_{d}\|\,\rho\left\|x^{\textrm{com}}[mN_{d}]\right\|\nonumber\\
& && \leq L_{j}\,\beta\left(\left\|x^{\textrm{com}}[0]\right\|,mN_{d}\right), \label{state_norm_ineq_01}
\end{alignat}
where $L_{j}:=\|A_{d}^{j}\|+\rho\sum_{\ell=0}^{j-1}\|A_{d}^{j-1-\ell}\|\|B_{d}\|$ for $j=0,1,\cdots,N_{d}-1$. By defining, $L:=\max\{L_{j}\,|\,0\leq{j}\leq{}N_{d}-1\}$, inequality \eqref{state_norm_ineq_01} becomes
\begin{equation}\label{state_norm_ineq_02}
\left\|x^{\textrm{com}}[mN_{d}+j]\right\|\leq{}L\,\beta\left(\left\|x^{\textrm{com}}[0]\right\|,mN_{d}\right).
\end{equation}
For a given $\epsilon>0$, one can choose $\delta>0$ such that $L\,\beta(\delta,0)<\epsilon$. Then, from \eqref{state_norm_ineq_02} and properties of class $\mathcal{KL}$ functions, $\|x^{\textrm{com}}[mN_{d}+j]\|\leq{}L\,\beta(\|x^{\textrm{com}}[0]\|,0)<\epsilon$ for every $\|x^{\textrm{com}}[0]\|<\delta$ and all $m\in\Integer$ and $j=0,1,\cdots,N_{d}-1$ which concludes the stability. From \eqref{state_norm_ineq_02}, $\lim_{m\rightarrow\infty}\|x^{\textrm{com}}[mN_{d}+j]\|=0$ for all $\|x^{\textrm{com}}[0]\|<\delta$ which completes the proof of asymptotic stability.
\end{proof}

\begin{remark}
Analogous to the Poincar\'e sections analysis that translates the stability of periodic trajectories into that of a fixed point for a discrete-time system refereed to as the Poincar\'e return map \cite{Grizzle_Asymptotically_Stable_Walking_IEEE_TAC,Hamed_Gregg_IEEE_TAC,Veer_Poulakakis}, Theorem \ref{stability_theorem} translates the stability of the target point for the closed-loop discrete-time system into that of the down-sample system. Similar to the Poincar\'e return map, the state of the down-sample system is updated in an event-based manner.
\end{remark}


\section{QP-BASED HZD CONTROLLERS}
\label{QP-BASED VIRTUAL CONSTRAINTS FOR TRACKING AND FEASIBILITY}
\vspace{-0.2em}

The objective of this section is to develop the low-level control algorithm based on QP and the virtual constraints approach. In particular, we would like to develop a nonlinear control algorithm to track the optimal COM trajectory that is generated by the higher-level MPC for the current continuous-time domain. The low-level controller also imposes the swing legs to follow an appropriate path to land at the desired footholds. The trajectory tracking problem is formulated via the virtual constraints approach. Since the higher-level LIP model only considers the feasibility of the net force $F_{\textrm{net}}$, the lower-level controller formulates the I-O linearization problem as a QP that addresses the feasibility of all individual contact forces $F_{\ell}$ for $\ell\in\mathcal{C}$ while tracking the LIP trajectories.

\noindent\textbf{Virtual Constraints:} In this paper, we define a set of time-varying and holonomic virtual constraints  as follows
\begin{equation}\label{virtual_constraints}
y(x,t):=h(q,t):=h_{0}(q)-h_{d}\left(s,\alpha\right),
\end{equation}
in which $h_{0}(q)$ denotes a set of holonomic quantities to be controlled. In addition, $h_{d}(s,\alpha)$ represents the desired evolution of the controlled variables on the gait in terms of the phasing variable $s$. Here, $s:=\frac{t-t^{+}}{N_{d}\,T_{d}}$
denotes the phasing variable with $t^{+}$ being the initial time for the current domain and $N_{d}T_{d}$ representing an estimated elapsed time for the domain. The desired trajectory $h_{d}(s,\alpha)$ is taken as a B\'ezier polynomial with a  coefficient matrix $\alpha$. During the quadruple-contact domains (i.e., four legs on the ground), we choose $h_{0}(q)\in\Real^{6}$ as the roll, pitch, and yaw angles of the torso together with the COM positions. The idea is to regulate the absolute orientation of the robot while imposing the COM coordinates to follow the optimal COM trajectory generated by the higher-level MPC. Here, the coefficient matrix $\alpha$ can be chosen via least squares at the beginning of each domain such that $h_{d}(s,\alpha)$ has the best fit to the optimal COM trajectory over $N_{d}$ samples. For double- and triple-contact domains, $h_{0}(q)$ is augmented with the Cartesian coordinates of the swing leg ends for foot placement. The idea is to follow a desired foot trajectory in the workspace starting from the previous foothold and ending at the next preplanned foothold. This makes the output function $12$- and $9$-dimensional for the double- and triple-contact domains, respectively. To control the configuration of the manipulator, we augment $h_{0}(q)$ and $h_{d}(s,\alpha)$ by the Cartesian coordinates of the end-effector and its desired trajectory in the workspace, respectively.

\noindent\textbf{QP-Based I-O Linearization:} Differentiating the output functions \eqref{virtual_constraints} along the trajectories of the full-order model \eqref{state_eq} results in the following output dynamics
\begin{alignat}{4}
&\ddot{y}&&=\textrm{L}_{g}\textrm{L}_{f}y(x,t)\,\tau+\textrm{L}_{w}\textrm{L}_{f}y(x,t)\,F+\textrm{L}_{f}^{2}y(x,t)+\frac{\partial^{2}y}{\partial{}t^{2}}(x,t)\nonumber\\
& &&=-K_{D}\dot{y} - K_{P}\,y,\label{output_dynamics}
\end{alignat}
in which $\textrm{L}_{g}\textrm{L}_{f}y:=\frac{\partial h}{\partial q}D^{-1}\Upsilon$, $\textrm{L}_{w}\textrm{L}_{f}y:=\frac{\partial h}{\partial q}D^{-1}J$, and $\textrm{L}_{f}^{2}y:=-\frac{\partial h}{\partial q}D^{-1}H+\frac{\partial}{\partial q}(\frac{\partial h}{\partial q}\dot{q})\,\dot{q}$ are Lie derivatives with $J$ being the total Jacobian matrix which satisfies $J^\top{}F=\sum_{\ell\in\mathcal{C}}J_{\ell}^\top{}F_{\ell}$. Moreover, $K_{P}$ and $K_{D}$ are positive definite matrices, and hence, $(y,\dot{y})=(0,0)$ is exponentially stable for the output dynamics \eqref{output_dynamics}. To compute the required torques that drive the outputs to the zero, one would need to solve for $\tau$ from \eqref{output_dynamics}. However, since the contact force measurements are not available for the studied robot, one would need to estimate the contact forces $F=\textrm{col}\{F_{\ell}|\,\ell\in\mathcal{C}\}$. We assume a rigid contact model with the walking surface. The robustness of the controller against uncertainty in the contact models will be investiagted in Section \ref{NUMERICAL SIMULATIONS}. This assumption makes the leg ends acceleration zero which can be expressed as follows:
\begin{equation}\label{acceleration_stance_leg_01}
J(q)\,\ddot{q}+\frac{\partial}{\partial q}\left(J(q)\,\dot{q}\right)\dot{q}=0.
\end{equation}
Next let us suppose that $p^{\textrm{st}}_{\ell}(q)\in\Real^{3}$ for $\ell\in\mathcal{C}$ denotes the Cartesian coordinates of the contacting foot $\ell$. The augmented stance feet coordinates can also be defined as $p^{\textrm{st}}(q):=\textrm{col}\{p^{\textrm{st}}_{\ell}|\,\ell\in\mathcal{C}\}$ which has the property $\frac{\partial p^{\textrm{st}}}{\partial q}(q)=J(q)$. Equation \eqref{acceleration_stance_leg_01} together with \eqref{state_eq} then yields
\begin{equation}\label{acceleration_stance_leg_02}
\ddot{p}^{\textrm{st}}=\textrm{L}_{g}\textrm{L}_{f}p^{\textrm{st}}(x)\,\tau+\textrm{L}_{w}\textrm{L}_{f}p^{\textrm{st}}(x)\,F+\textrm{L}_{f}^{2}p^{\textrm{st}}(x)=0,
\end{equation}
where $\textrm{L}_{g}\textrm{L}_{f}p^{\textrm{st}}:=J\,D^{-1}\Upsilon$, $\textrm{L}_{w}\textrm{L}_{f}p^{\textrm{st}}:=J\,D^{-1}J^\top$, and $\textrm{L}_{f}^{2}p^{\textrm{st}}:=-J\,D^{-1}H+\frac{\partial}{\partial q}(J\,\dot{q})\,\dot{q}$. Now, we need to look for the values of $(\tau,F)$ that satisfy \eqref{output_dynamics} and \eqref{acceleration_stance_leg_02} with contact forces being in the friction cone, that is, $F_{\ell}\in\mathcal{FC}$ for all $\ell\in\mathcal{C}$. For this purpose, we set up the following real-time QP
\begin{alignat}{8}
& &&\min_{(\tau,F,\omega)}  &&\quad \frac{1}{2}\|\tau\|_{2}^{2}+\frac{\gamma}{2}\|\omega\|^{2}_{2}\nonumber\\
& &&\,\,\textrm{s.t.} && \textrm{L}_{g}\textrm{L}_{f}y\,\tau+\textrm{L}_{w}\textrm{L}_{f}y\,F+\textrm{L}_{f}^{2}y+\frac{\partial^2y}{\partial t^2}+\omega=v_{\textrm{PD}}(y,\dot{y})\nonumber\\
& && && \textrm{L}_{g}\textrm{L}_{f}p^{\textrm{st}}\,\tau+\textrm{L}_{w}\textrm{L}_{f}p^{\textrm{st}}\,F+\textrm{L}_{f}^{2}p^{\textrm{st}}=0\nonumber\\
& && && F_{\ell}\in\mathcal{FC},\quad\ell\in\mathcal{C}\nonumber\\
& && && \tau_{\min}\leq\tau\leq\tau_{\max},\label{low_levl_QP}
\end{alignat}
in which $v_{\textrm{PD}}(y,\dot{y}):=-K_{P}\,y-K_{D}\,\dot{y}$ is the PD action. The equality constraints for the QP are set up based on the output dynamics \eqref{output_dynamics} as well as the stance foot accelerations assumption \eqref{acceleration_stance_leg_02}. In case the decoupling matrix $\textrm{L}_{g}\textrm{L}_{f}y$ becomes singular, there is a possibility for the infeasibility of  \eqref{output_dynamics} and \eqref{acceleration_stance_leg_02}. To tackle this issue, we introduce a defect variable $\omega$ in the equality constraint of QP \eqref{low_levl_QP} that corresponds to \eqref{output_dynamics}. To minimize the effect of the defect variable, we then add a quadratic term $\frac{\gamma}{2}\|\omega\|_{2}^{2}$ to the cost function of the QP to ensure that the $2$-norm of the defect variable is as small as possible. Here, $\gamma>0$ is a weighting factor. The other term in the cost function tries to find the minimum $2$-norm (minimum power) torques that satisfy the equality and inequality constraints. Furthermore, $\tau_{\min}$ and $\tau_{\max}$ are the admissible lower and upper bounds on the torques, respectively. The optimal solution of \eqref{low_levl_QP} (i.e., $\tau$) is finally employed to the full-order system.


\section{NUMERICAL SIMULATIONS}
\label{NUMERICAL SIMULATIONS}
\vspace{-0.2em}

The objective of this section is to numerically verify the theoretical results of the paper. To demonstrate the power and robustness of the proposed hierarchical control algorithm, we consider the full-order dynamical model of the composite robot in two different simulations. The first simulation is implemented in MATLAB/Simulnik with LuGre contact models \cite{LuGre_model} (compliant contact model), whereas the second one utilizes RaiSim \cite{RAISIM} (rigid contact model). We consider trot gaits with five different directions (i.e., forward, backward, sideways, diagonal, and in-place) whose graphs $\mathcal{G}$ consist of $M=20$ continuous-time domains (see Fig. \ref{HierachicalControl_MPC_illustration}). The first and last domains of the graph are assumed to be quadruple-contact domains for starting and stopping the gait, whereas the intermediate domains are supposed to be double-contact domains. We choose the sampling time to discretize the LIP dynamics as $T_{d}=80$ (ms) with $N_{d}=4$ grids per each domain. The control horizon for the event-based MPC is chosen as $N=nN_{d}=8$ which considers two domains ahead. The other parameters are taken as $P=10^{3}\,I_{4\times4}$, $Q=I_{4\time4}$, $R=I_{2\times2}$,a nd $\hat{R}=0.01I$ that stabilize the target point for the down-sample discrete-time system \eqref{downsample_closed_loop_LIP}. With the height of the COM being $0.5$ (m) and the friction coefficient $\mu=0.4$, the higher-level QP (i.e., MPC) is solved in an event-based manner, that is  approximately every $N_{d}T_{d}=0.32$ seconds. Since the dimension of $\lambda_{k+i|k}$ changes per domain, the number of decision variables for the MPC depends on the domain number. In particular, it can be shown that for $\zeta\in\{2,\cdots,M-2\}$, $\zeta\in\{1,M-1\}$, and $\zeta=M$, the higher-level QP has $64$, $72$, and $80$ decision variables, respectively. In MATLAB/Simulink, we make use of the ECOS QP \cite{ecos} solver, whereas in RaiSim, we use qpSWIFT \cite{qpSWIFT}. The lower-level QP for I-O linearization has $37$ decision variables for both double- and quadruple-contact domains, which is approximately $50\%$ of the number of decisions variables used for the higher-level QP. The lower-level QP is solved with the weighting factor $\gamma=10^7$ at every $1$ (ms). All state components of the robot, except the absolute Cartesian coordinates of the torso (i.e., $p_{b})$, are measurable by an inertial measurement unit as well as encoders. Since some components of the controlled variables, $h_{0}(q)$, are defined in the workspace, one would only need to estimate $p_{b}$. Analogous to the approach of \cite{KF_mitcheetah3}, we utilize a Kalman filter for this purpose. Figure \ref{reduced_order_steering} depicts the evolution of the COM and COP for the forward and diagonal trot gaits of the closed-loop discrete-time dynamics \eqref{closed_loop_LIP} with the step lengths of $(10,0)$ (cm) and $(7,4)$ (cm) in $\Real^{2}$, respectively, versus the discrete-time $k$. Convergence to the target state is clear. Figure \ref{full_order_simulations} illustrates the evolution of the output functions and torque inputs (i.e., before the gear ratio) for the full-order dynamical model of the forward trot gait with the maximum speed of $0.3$ (m/s) in two different simulations. Regardless of the difference in the contact models, the robot travels in a stable and robust manner towards the target. Figure \ref{GRF_plot_delayedsystem} depicts the GRF of one of the contacting legs with the walking surface. To demonstrate the robustness of the proposed control algorithm against the control frequency as well as time delays, we next assume that the low-level control frequency is reduced from 1 kHz to 500 Hz while there is a latency of 2 (ms) in solving QPs. Figure \ref{GRF_plot_delayedsystem} illustrates the virtual constraints profile for this case. It is clear that the outputs are still driven to zero while being in a feasible range. Figure \ref{COM_traj_foothholds} finally depicts the plot of the COM trajectory for the full-order model as well as the footholds in the $xy$-plane for different gaits. Animations of these simulations and other gaits can be found at \cite{YouTube_EventBasedMPC_QPVirtualConstraints}.

\begin{figure}[t!]
\centering
\vspace{0.1em}
\includegraphics[width=\linewidth]{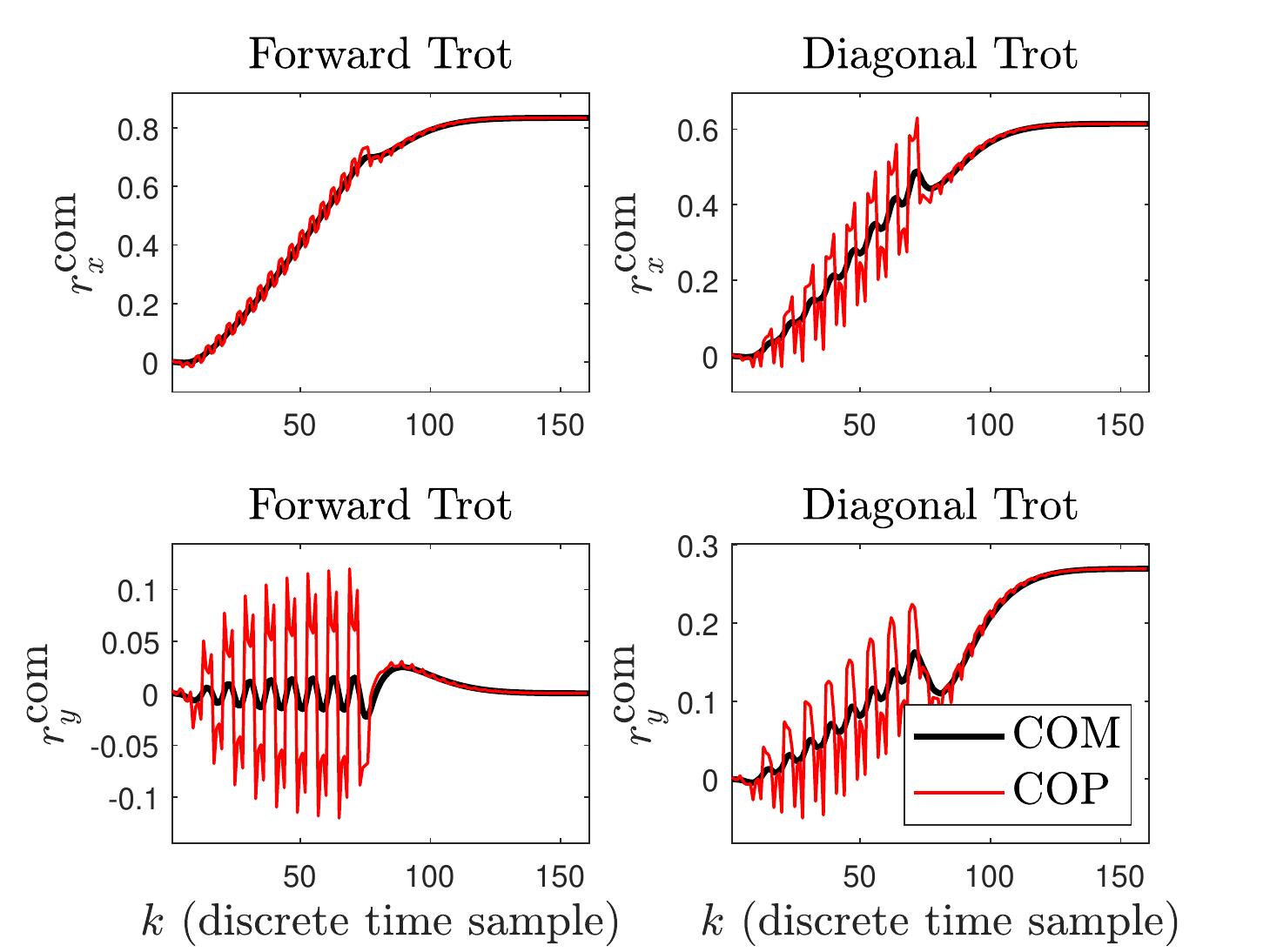}
\caption{Plot of the COM and COP trajectories for the forward and diagonal trot gaits of the closed-loop reduced-order system \eqref{closed_loop_LIP} versus $k$. Convergence to the target point is clear.}
\vspace{-1.0em}
\label{reduced_order_steering}
\end{figure}

\begin{figure}[t!]
\centering
\subfloat{\includegraphics[width=1.65in]{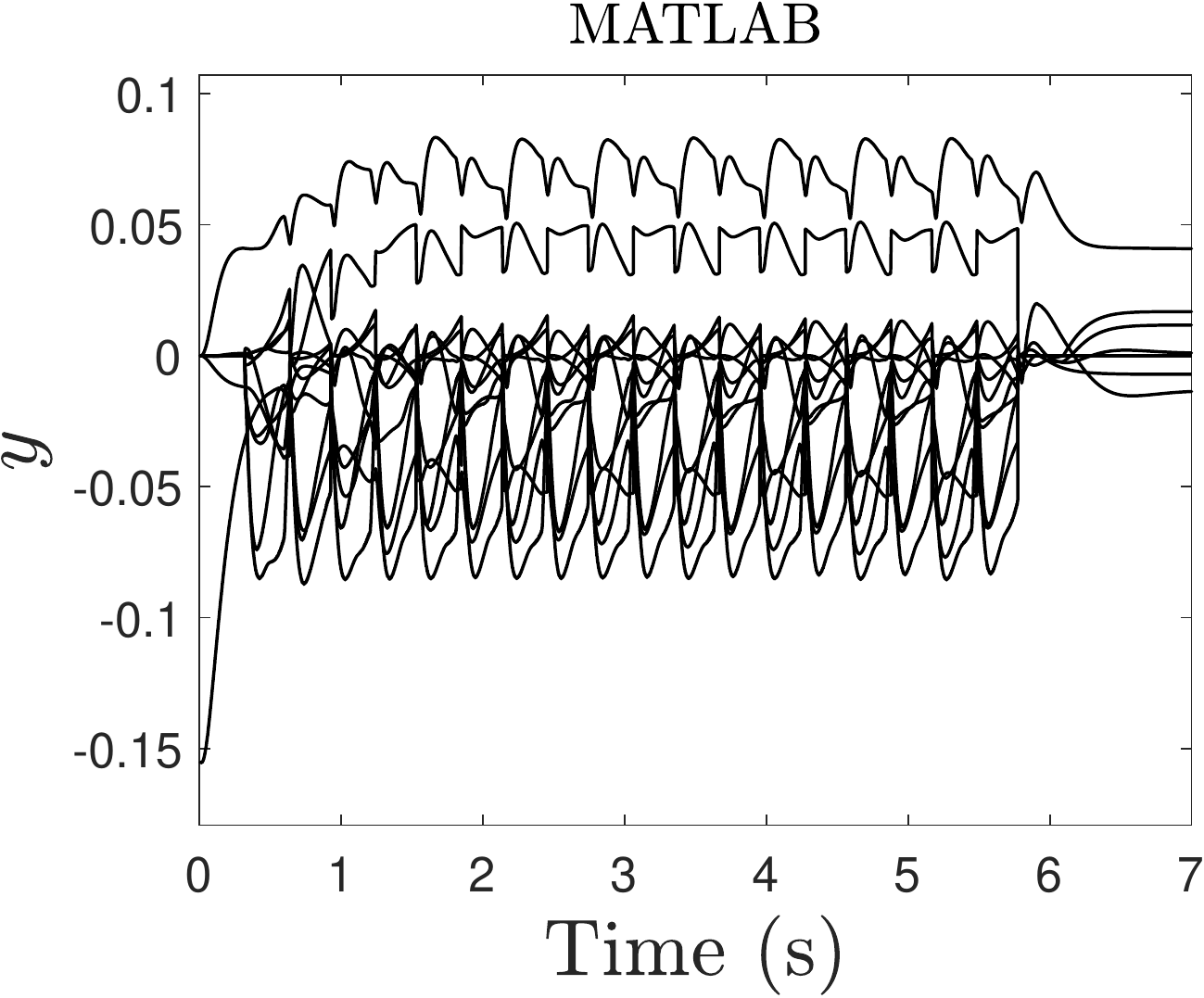}\label{}}\,
\subfloat{\includegraphics[width=1.65in]{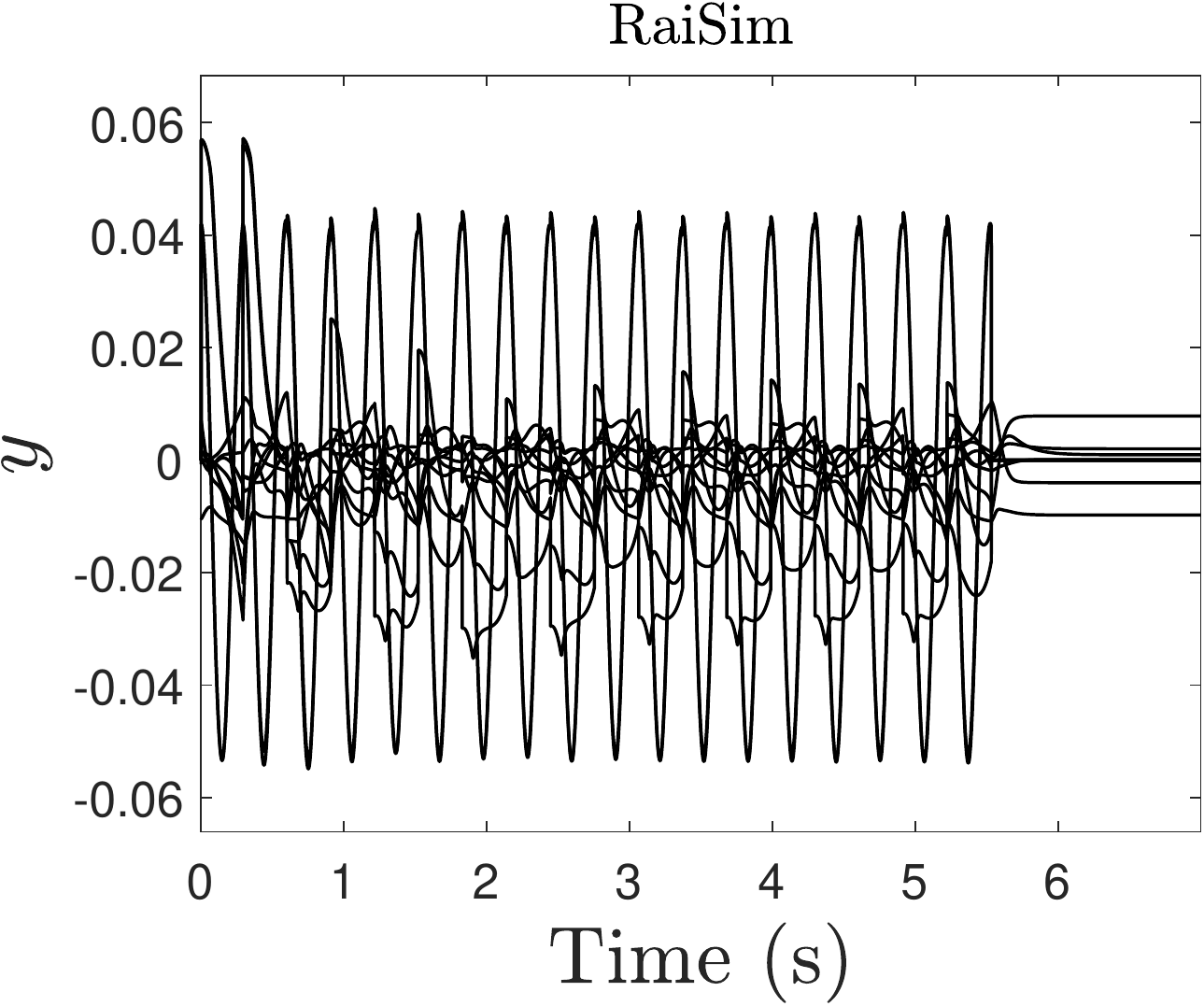}\label{}}\\
\vspace{-0.2em}
\subfloat{\includegraphics[width=1.65in]{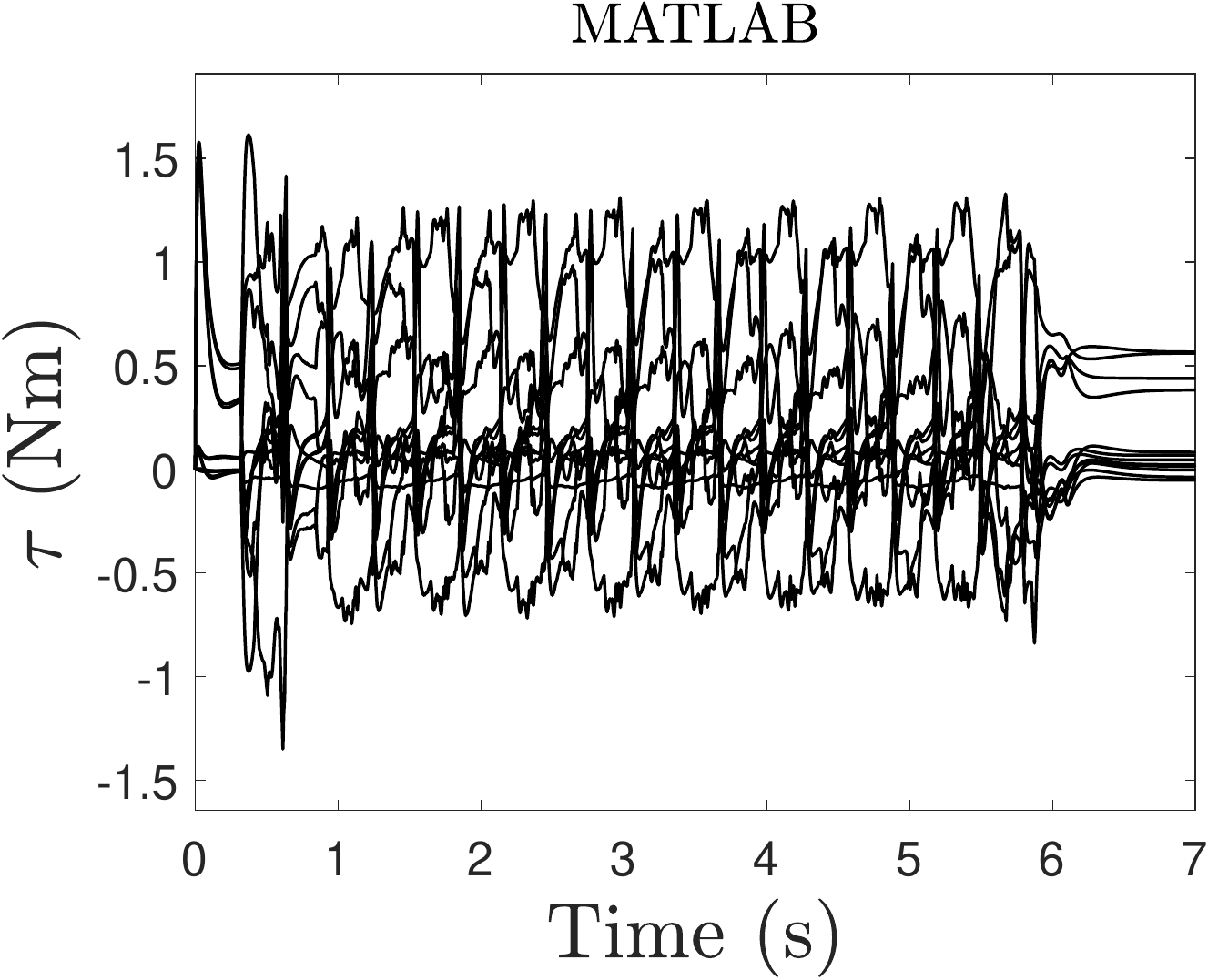}\label{}}\,
\subfloat{\includegraphics[width=1.65in]{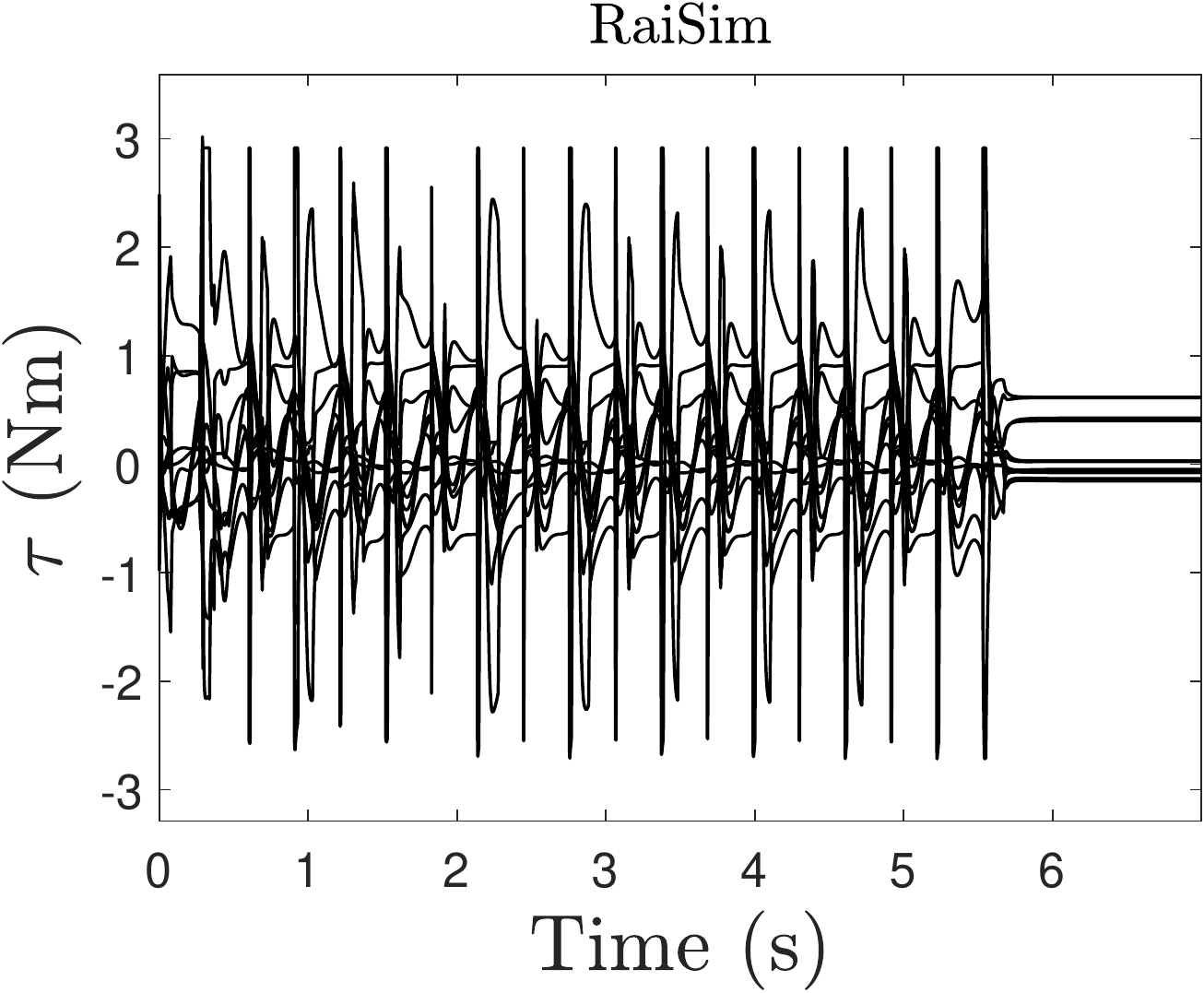}\label{}}
\caption{Plot of the virtual constraints and torque inputs for forward trot gait of the full-order closed-loop system in MATLAB/Simulink and RaiSim with $M=20$ domains.}
\vspace{-1.75em}
\label{full_order_simulations}
\end{figure}

\begin{figure}[t!]
\centering
\vspace{0.1em}
\subfloat{\includegraphics[width=1.7in]{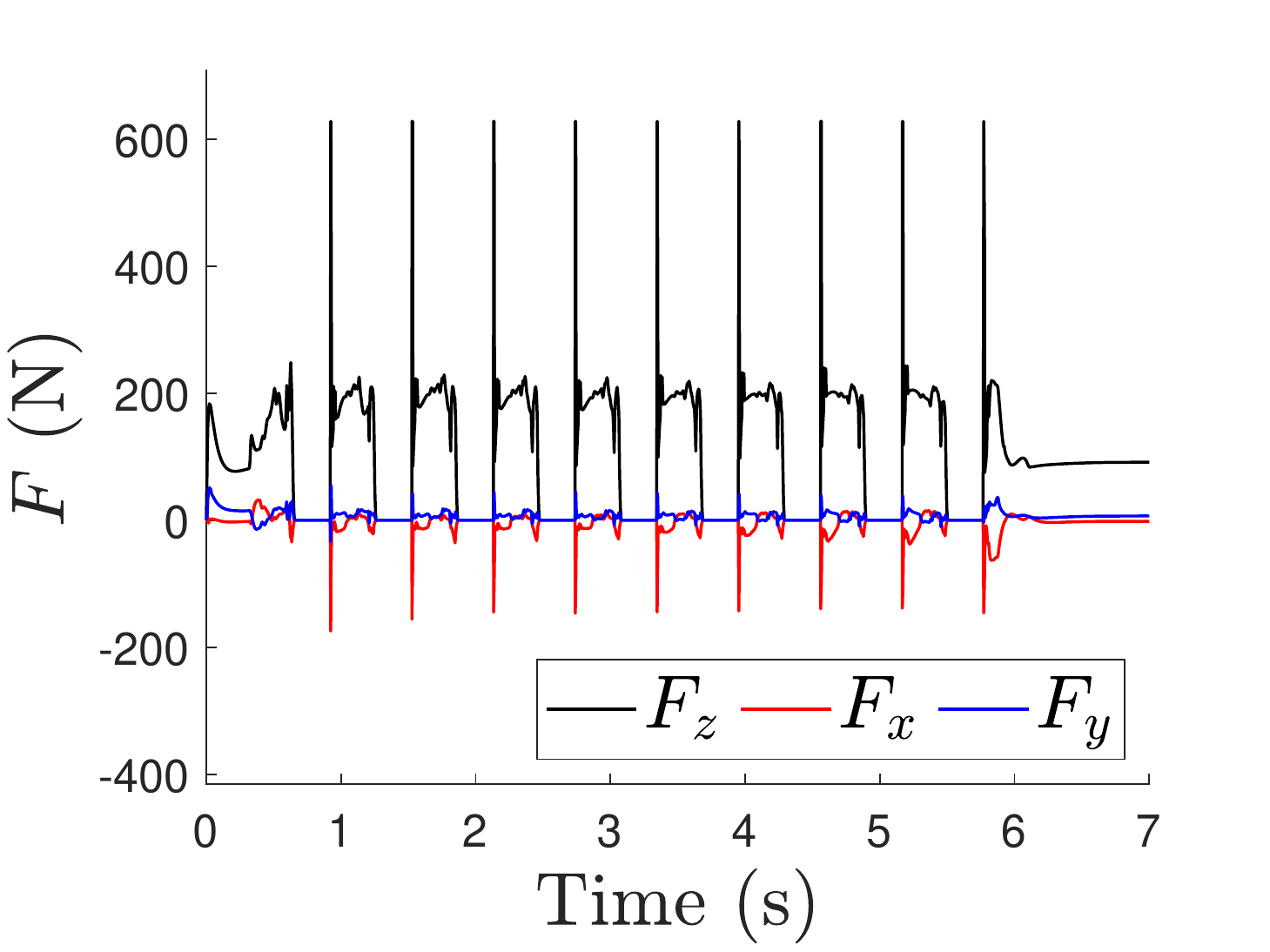}\label{}}
    \subfloat{\includegraphics[width=1.7in]{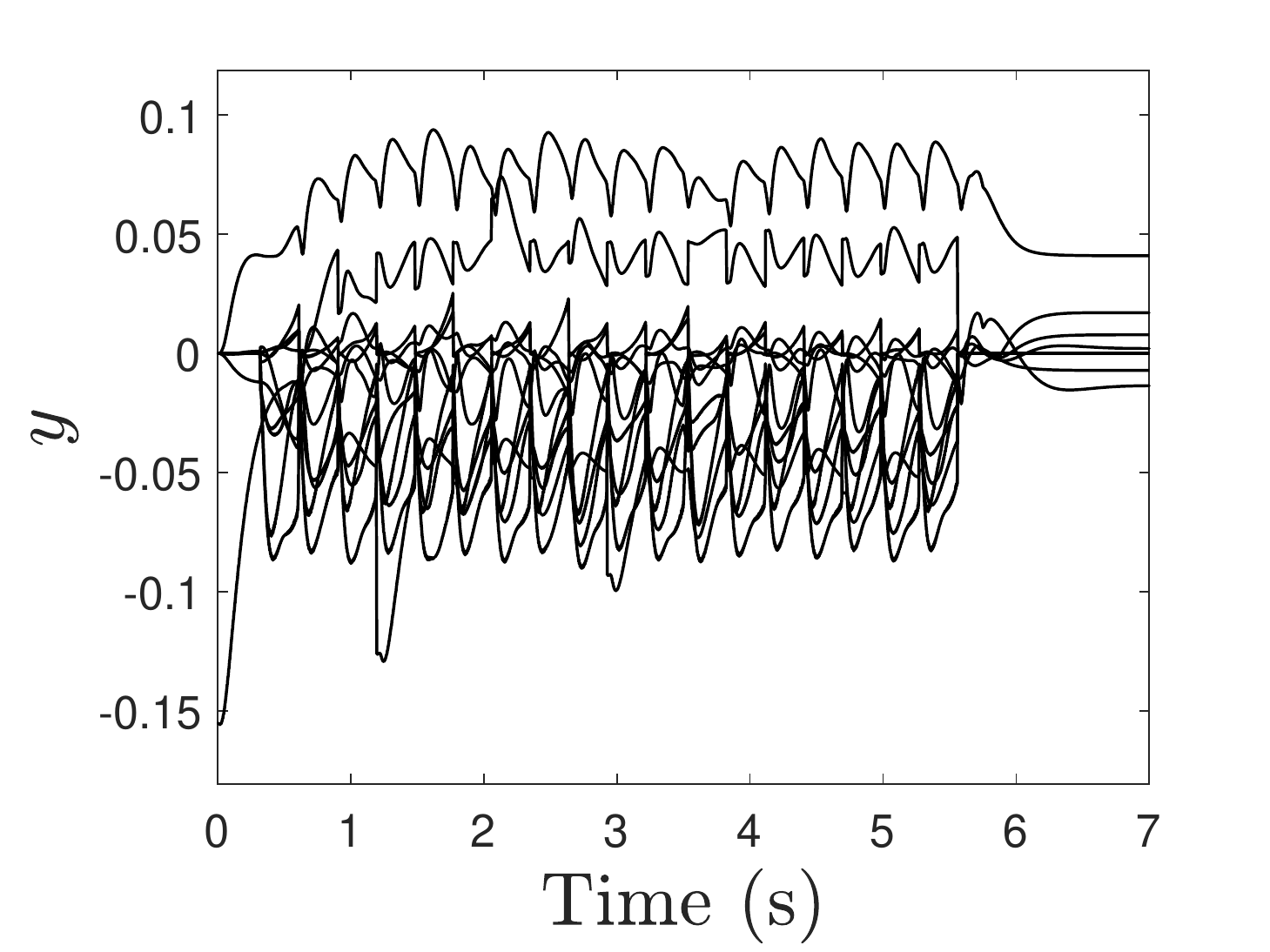}\label{}}
\caption{Left: Plot of the GRF on the right front leg using the LuGre contact model. Right: Plot of the virtual constraints when the low-level QP-based nonlinear control is solved at 500 Hz with a delay of 2 (ms).}
\vspace{-1.5em}
\label{GRF_plot_delayedsystem}
\end{figure}

\begin{figure}[t!]
\centering
\subfloat{\includegraphics[width=1.7in]{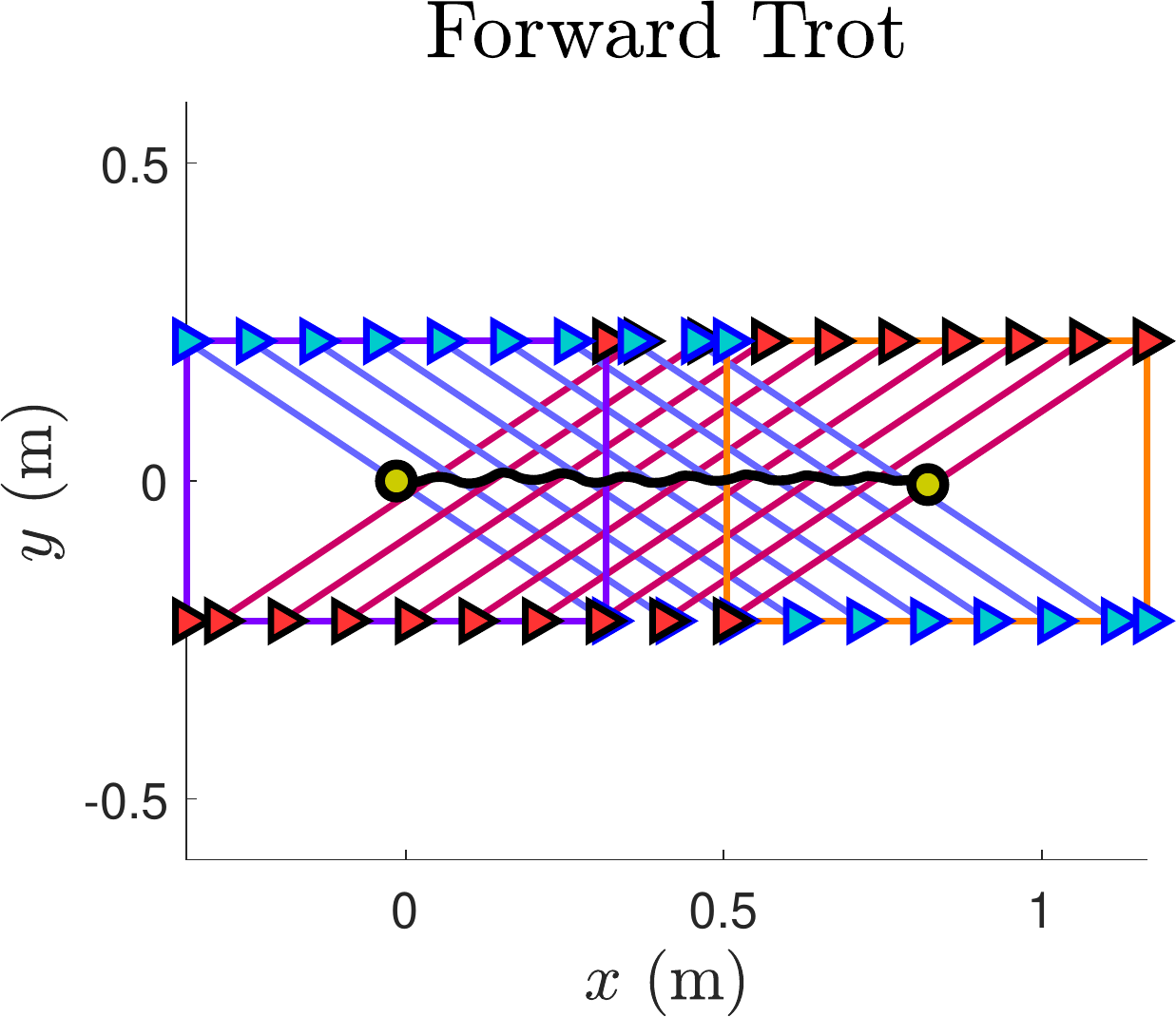}\label{}}
\subfloat{\includegraphics[width=1.7in]{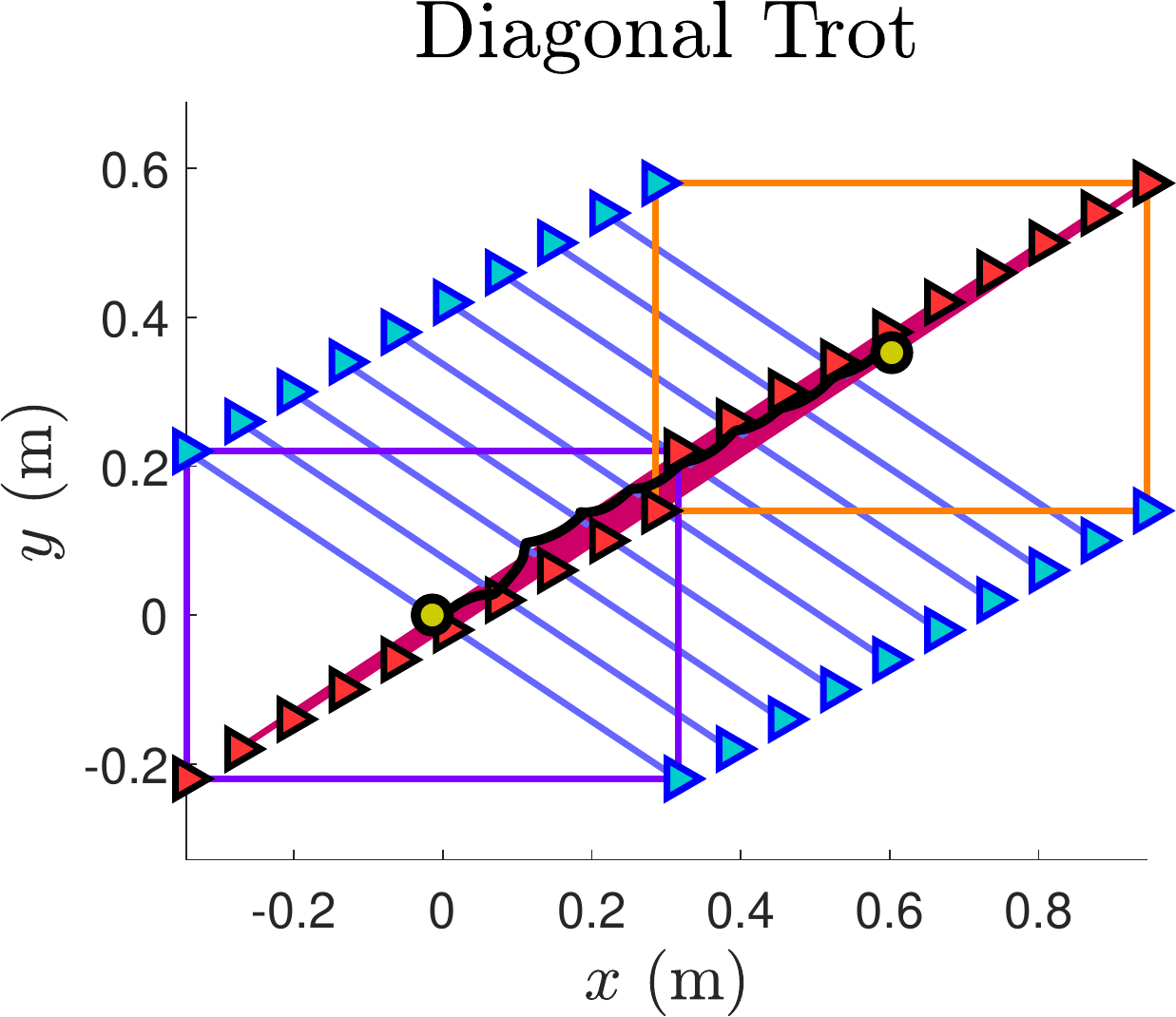}\label{}}\\
\caption{Plot of the COM trajectory and footholds in the $xy$-plane for different trot gaits. The solid trajectory and circles denote the the COM trajectory and initial and final COM positions, respectively. The lines represent the support polygon $\mathcal{U}_{\zeta}$ for each domain. The triangles illustrate the desired footholds.}
\vspace{-1.5em}
\label{COM_traj_foothholds}
\end{figure}


\vspace{-0.5em}
\section{CONCLUSION}
\label{CONCLUSION}
\vspace{-0.2em}

This paper proposed a hierarchical control scheme for quadrupedal locomotion based on convex optimization, event-based MPC, and virtual constraints. At the higher level of the control scheme, the event-based MPC computes the optimal trajectory for the COM of the LIP model to steer the robot from an initial state to a final state while the net GRF is feasible. The paper addressed the asymptotic stability of the target point under the proposed MPC approach. It was formally shown that one would \textit{not} need to employ the MPC at every time sample to stabilize the locomotion. The MPC can instead be employed in an event-based manner at the beginning of each domain to stabilize the target point. This significantly reduces the complexity for real-time implementation of MPC techniques. The lower-level controller then fills the gap between the reduced- and full-order dynamics. In particular, we formulated a QP-based I-O linearization approach to impose the full-order dynamics to follow the optimal COM trajectory of the reduced-order model while tracking the desired footholds with all individual GRFs being in the friction cone. The power and robustness of the proposed hierarchical control scheme were demonstrated via numerical simulations of a full-order dynamical model for a 22 DOF quadrupedal robot with different contact models. The framework can systematically address a range of aperiodic locomotion patterns such as forward, backward, lateral, diagonal, and in-place gaits with start and stop conditions.

For future work, we will experimentally investigate the proposed control scheme on the Vision 60 platform. We will also generalize the framework to develop distributed controllers for motion control of collaborative quadrupedal robots.


\bibliographystyle{IEEEtran}
\bibliography{references}

\begin{thebibliography}{10}
\providecommand{\url}[1]{#1}
\csname url@samestyle\endcsname
\providecommand{\newblock}{\relax}
\providecommand{\bibinfo}[2]{#2}
\providecommand{\BIBentrySTDinterwordspacing}{\spaceskip=0pt\relax}
\providecommand{\BIBentryALTinterwordstretchfactor}{4}
\providecommand{\BIBentryALTinterwordspacing}{\spaceskip=\fontdimen2\font plus
\BIBentryALTinterwordstretchfactor\fontdimen3\font minus
  \fontdimen4\font\relax}
\providecommand{\BIBforeignlanguage}[2]{{%
\expandafter\ifx\csname l@#1\endcsname\relax
\typeout{** WARNING: IEEEtran.bst: No hyphenation pattern has been}%
\typeout{** loaded for the language `#1'. Using the pattern for}%
\typeout{** the default language instead.}%
\else
\language=\csname l@#1\endcsname
\fi
#2}}
\providecommand{\BIBdecl}{\relax}
\BIBdecl

\bibitem{Grizzle_Asymptotically_Stable_Walking_IEEE_TAC}
J.~Grizzle, G.~Abba, and F.~Plestan, ``Asymptotically stable walking for biped
  robots: {A}nalysis via systems with impulse effects,'' \emph{IEEE
  Transactions on Automatic Control}, vol.~46, no.~1, pp. 51--64, Jan 2001.

\bibitem{Ames_RES_CLF_IEEE_TAC}
A.~Ames, K.~Galloway, K.~Sreenath, and J.~Grizzle, ``Rapidly exponentially
  stabilizing control {Lyapunov} functions and hybrid zero dynamics,''
  \emph{IEEE Transactions on Automatic Control}, vol.~59, no.~4, pp. 876--891,
  April 2014.

\bibitem{Sreenath_Grizzle_HZD_Walking_IJRR}
K.~Sreenath, H.-W. Park, I.~Poulakakis, and J.~W. Grizzle, ``Compliant hybrid
  zero dynamics controller for achieving stable, efficient and fast bipedal
  walking on {MABEL},'' \emph{The International Journal of Robotics Research},
  vol.~30, no.~9, pp. 1170--1193, Aug. 2011.

\bibitem{Veer_Poulakakis}
S.~{Veer}, {Rakesh}, and I.~{Poulakakis}, ``Input-to-state stability of
  periodic orbits of systems with impulse effects via {Poincar\'e} analysis,''
  \emph{IEEE Transactions on Automatic Control}, vol.~64, no.~11, pp.
  4583--4598, Nov 2019.

\bibitem{Tedrake_Robus_Limit_Cycles_CDC}
H.~Dai and R.~Tedrake, ``Optimizing robust limit cycles for legged locomotion
  on unknown terrain,'' in \emph{IEEE 51st Annual Conference on Decision and
  Control}, Dec 2012, pp. 1207--1213.

\bibitem{Byl_HZD}
C.~O. Saglam and K.~Byl, ``Meshing hybrid zero dynamics for rough terrain
  walking,'' in \emph{2015 IEEE International Conference on Robotics and
  Automation (ICRA)}, May 2015, pp. 5718--5725.

\bibitem{Johnson_Burden_Koditschek}
A.~M. Johnson, S.~A. Burden, and D.~E. Koditschek, ``A hybrid systems model for
  simple manipulation and self-manipulation systems,'' \emph{The International
  Journal of Robotics Research}, vol.~35, no.~11, pp. 1354--1392, 2016.

\bibitem{Spong_Controlled_Symmetries_IEEE_TAC}
M.~Spong and F.~Bullo, ``Controlled symmetries and passive walking,''
  \emph{IEEE Transactions on Automatic Control}, vol.~50, no.~7, pp.
  1025--1031, July 2005.

\bibitem{Manchester_Tedrake_LQR_IJRR}
I.~Manchester, U.~Mettin, F.~Iida, and R.~Tedrake, ``Stable dynamic walking
  over uneven terrain,'' \emph{The International Journal of Robotics Research},
  vol.~30, no.~3, pp. 265--279, 2011.

\bibitem{Vasudevan2017}
R.~Vasudevan, \emph{Hybrid System Identification via Switched System Optimal
  Control for Bipedal Robotic Walking}.\hskip 1em plus 0.5em minus 0.4em\relax
  Cham: Springer International Publishing, 2017, pp. 635--650.

\bibitem{Hamed_Gregg_IEEE_TAC}
K.~Akbari~Hamed and R.~D. Gregg, ``Decentralized event-based controllers for
  robust stabilization of hybrid periodic orbits: Application to underactuated
  {3D} bipedal walking,'' \emph{IEEE Transactions on Automatic Control},
  vol.~64, no.~6, pp. 2266--2281, June 2019.

\bibitem{Hamed_Buss_Grizzle_BMI_IJRR}
K.~Akbari~Hamed, B.~Buss, and J.~Grizzle, ``Exponentially stabilizing
  continuous-time controllers for periodic orbits of hybrid systems:
  Application to bipedal locomotion with ground height variations,'' \emph{The
  International Journal of Robotics Research}, vol.~35, no.~8, pp. 977--999,
  2016.

\bibitem{Ames_HybridReduction_Original_Paper}
A.~D. Ames, R.~D. Gregg, E.~D.~B. Wendel, and S.~Sastry, ``On the geometric
  reduction of controlled three-dimensional bipedal robotic walkers,'' in
  \emph{Lagrangian and Hamiltonian Methods for Nonlinear Control 2006}.\hskip
  1em plus 0.5em minus 0.4em\relax Berlin, Heidelberg: Springer Berlin
  Heidelberg, 2007, pp. 183--196.

\bibitem{Westervelt_Grizzle_Koditschek_HZD_IEEE_TRO}
E.~Westervelt, J.~Grizzle, and D.~Koditschek, ``Hybrid zero dynamics of planar
  biped walkers,'' \emph{IEEE Transactions on Automatic Control}, vol.~48,
  no.~1, pp. 42--56, Jan 2003.

\bibitem{Isidori_Book}
A.~Isidori, \emph{Nonlinear Control Systems}.\hskip 1em plus 0.5em minus
  0.4em\relax Springer; 3rd edition, 1995.

\bibitem{Cheavallereau_Grizzle_RABBIT}
C.~Chevallereau, G.~Abba, Y.~Aoustin, F.~Plestan, E.~Westervelt, C.~Canudas-de
  Wit, and J.~Grizzle, ``{RABBIT}: {A} testbed for advanced control theory,''
  \emph{IEEE Control Systems Magazine}, vol.~23, no.~5, pp. 57--79, Oct 2003.

\bibitem{Ramezani_Hurst_Hamed_Grizzle_ATRIAS_ASME}
A.~Ramezani, J.~Hurst, K.~Akbai~Hamed, and J.~Grizzle, ``Performance analysis
  and feedback control of {ATRIAS}, a three-dimensional bipedal robot,''
  \emph{Journal of Dynamic Systems, Measurement, and Control December, ASME},
  vol. 136, no.~2, December 2013.

\bibitem{Ames_DURUS_TRO}
A.~Hereid, C.~M. Hubicki, E.~A. Cousineau, and A.~D. Ames, ``Dynamic humanoid
  locomotion: A scalable formulation for {HZD} gait optimization,'' \emph{IEEE
  Transactions on Robotics}, pp. 1--18, 2018.

\bibitem{Martin_Schmiedeler_IJRR}
A.~E. Martin, D.~C. Post, and J.~P. Schmiedeler, ``{The effects of foot
  geometric properties on the gait of planar bipeds walking under HZD-based
  control},'' \emph{The International Journal of Robotics Research}, vol.~33,
  no.~12, pp. 1530--1543, 2014.

\bibitem{Gregg_Toward_Biomimetic_Control_IEEE_CST}
R.~Gregg and J.~Sensinger, ``Towards biomimetic virtual constraint control of a
  powered prosthetic leg,'' \emph{IEEE Transactions on Control Systems
  Technology}, vol.~22, no.~1, pp. 246--254, Jan 2014.

\bibitem{zhao2016multicontact}
H.~Zhao, J.~Horn, J.~Reher, V.~Paredes, and A.~D. Ames, ``Multicontact
  locomotion on transfemoral prostheses via hybrid system models and
  optimization-based control,'' \emph{IEEE Transactions on Automation Science
  and Engineering}, vol.~13, no.~2, pp. 502--513, 2016.

\bibitem{agrawal2017first}
A.~Agrawal, O.~Harib, A.~Hereid, S.~Finet, M.~Masselin, L.~Praly, A.~D. Ames,
  K.~Sreenath, and J.~W. Grizzle, ``First steps towards translating hzd control
  of bipedal robots to decentralized control of exoskeletons,'' \emph{IEEE
  Access}, vol.~5, pp. 9919--9934, 2017.

\bibitem{Hamed_Ma_Ames_Vision60}
K.~{Akbari Hamed}, W.~Ma, and A.~D. Ames, ``Dynamically stable {3D} quadrupedal
  walking with multi-domain hybrid system models and virtual constraint
  controllers,'' in \emph{2019 American Control Conference (ACC)}, July 2019,
  pp. 4588--4595.

\bibitem{Leonessa_Pratt_MPC}
R.~J. {Griffin}, G.~{Wiedebach}, S.~{Bertrand}, A.~{Leonessa}, and J.~{Pratt},
  ``Walking stabilization using step timing and location adjustment on the
  humanoid robot, {Atlas},'' in \emph{2017 IEEE/RSJ International Conference on
  Intelligent Robots and Systems (IROS)}, Sep. 2017, pp. 667--673.

\bibitem{Ott_MPC}
J.~{Englsberger}, C.~{Ott}, M.~A. {Roa}, A.~{Albu-Schäffer}, and
  G.~{Hirzinger}, ``Bipedal walking control based on capture point dynamics,''
  in \emph{2011 IEEE/RSJ International Conference on Intelligent Robots and
  Systems}, Sep. 2011, pp. 4420--4427.

\bibitem{LIP_Original_Kajita}
S.~{Kajita}, F.~{Kanehiro}, K.~{Kaneko}, K.~{Fujiwara}, K.~{Harada},
  K.~{Yokoi}, and H.~{Hirukawa}, ``Biped walking pattern generation by using
  preview control of zero-moment point,'' in \emph{2003 IEEE International
  Conference on Robotics and Automation}, vol.~2, Sep. 2003, pp. 1620--1626
  vol.2.

\bibitem{Pratt_LIP}
J.~{Pratt}, J.~{Carff}, S.~{Drakunov}, and A.~{Goswami}, ``Capture point: {A}
  step toward humanoid push recovery,'' in \emph{2006 6th IEEE-RAS
  International Conference on Humanoid Robots}, Dec 2006, pp. 200--207.

\bibitem{little_dog_QP}
M.~{Kalakrishnan}, J.~{Buchli}, P.~{Pastor}, M.~{Mistry}, and S.~{Schaal},
  ``Fast, robust quadruped locomotion over challenging terrain,'' in \emph{2010
  IEEE International Conference on Robotics and Automation}, May 2010, pp.
  2665--2670.

\bibitem{Vukobratovic_Book}
M.~Vukobratovi\'c, B.~Borovac, and D.~Surla, \emph{Dynamics of Biped
  Locomotion}.\hskip 1em plus 0.5em minus 0.4em\relax Springer, 1990.

\bibitem{Kim_Wensing_Convex_MPC_01}
J.~{Di Carlo}, P.~M. {Wensing}, B.~{Katz}, G.~{Bledt}, and S.~{Kim}, ``Dynamic
  locomotion in the {MIT Cheetah 3} through convex model-predictive control,''
  in \emph{2018 IEEE/RSJ International Conference on Intelligent Robots and
  Systems (IROS)}, Oct 2018, pp. 1--9.

\bibitem{Park_Pandala_Ding_MPC_01}
Y.~{Ding}, A.~{Pandala}, and H.~{Park}, ``Real-time model predictive control
  for versatile dynamic motions in quadrupedal robots,'' in \emph{2019
  International Conference on Robotics and Automation (ICRA)}, May 2019, pp.
  8484--8490.

\bibitem{seminimpc}
O.~Villarreal, V.~Barasuol, P.~Wensing, and C.~Semini, ``{MPC}-based controller
  with terrain insight for dynamic legged locomotion,'' \emph{arXiv preprint
  arXiv:1909.13842}, 2019.

\bibitem{LuGre_model}
C.~De~Wit, H.~Olsson, K.~Astrom, and P.~Lischinsky, ``A new model for control
  of systems with friction,'' \emph{IEEE Transactions on Automatic Control},
  vol.~40, no.~3, pp. 419--425, Mar 1995.

\bibitem{RAISIM}
J.~{Hwangbo}, J.~{Lee}, and M.~{Hutter}, ``Per-contact iteration method for
  solving contact dynamics,'' \emph{IEEE Robotics and Automation Letters},
  vol.~3, no.~2, pp. 895--902, April 2018.

\bibitem{NLIP_LIP_HyQ}
A.~W. {Winkler}, F.~{Farshidian}, D.~{Pardo}, M.~{Neunert}, and J.~{Buchli},
  ``Fast trajectory optimization for legged robots using vertex-based {ZMP}
  constraints,'' \emph{IEEE Robotics and Automation Letters}, vol.~2, no.~4,
  pp. 2201--2208, Oct 2017.

\bibitem{kim2019highly}
D.~{Kim}, J.~D. {Carlo}, B.~{Katz}, G.~{Bledt}, and S.~{Kim}, ``Highly dynamic
  quadruped locomotion via whole-body impulse control and model predictive
  control,'' \emph{arXiv: 1909.06586}, 2019.

\bibitem{Jessy_Book}
E.~Westervelt, J.~Grizzle, C.~Chevallereau, J.~Choi, and B.~Morris,
  \emph{Feedback Control of Dynamic Bipedal Robot Locomotion}.\hskip 1em plus
  0.5em minus 0.4em\relax Taylor \& Francis/CRC, 2007.

\bibitem{NCP_LCP}
D.~E. Stewart and J.~C. Trinkle, ``An implicit time-stepping scheme for rigid
  body dynamics with inelastic collisions and coulomb friction,''
  \emph{International Journal for Numerical Methods in Engineering}, vol.~39,
  no.~15, pp. 2673--2691, 1996.

\bibitem{MuJoCo}
E.~{Todorov}, T.~{Erez}, and Y.~{Tassa}, ``{MuJoCo: A physics engine for
  model-based control},'' in \emph{2012 IEEE/RSJ International Conference on
  Intelligent Robots and Systems}, Oct 2012, pp. 5026--5033.

\bibitem{Boyd_FastMPC}
Y.~{Wang} and S.~{Boyd}, ``Fast model predictive control using online
  optimization,'' \emph{IEEE Transactions on Control Systems Technology},
  vol.~18, no.~2, pp. 267--278, March 2010.

\bibitem{Coroianu2016}
L.~Coroianu, ``Best {Lipschitz} constants of solutions of quadratic programs,''
  \emph{Journal of Optimization Theory and Applications}, vol. 170, no.~3, pp.
  853--875, Sep 2016.

\bibitem{ecos}
A.~{Domahidi}, E.~{Chu}, and S.~{Boyd}, ``{ECOS: An SOCP solver for embedded
  systems},'' in \emph{2013 European Control Conference (ECC)}, July 2013, pp.
  3071--3076.

\bibitem{qpSWIFT}
A.~G. {Pandala}, Y.~{Ding}, and H.~{Park}, ``{qpSWIFT}: {A} real-time sparse
  quadratic program solver for robotic applications,'' \emph{IEEE Robotics and
  Automation Letters}, vol.~4, no.~4, pp. 3355--3362, Oct 2019.

\bibitem{KF_mitcheetah3}
G.~{Bledt}, M.~J. {Powell}, B.~{Katz}, J.~{Di Carlo}, P.~M. {Wensing}, and
  S.~{Kim}, ``{MIT Cheetah 3}: {Design} and control of a robust, dynamic
  quadruped robot,'' in \emph{2018 IEEE/RSJ International Conference on
  Intelligent Robots and Systems (IROS)}, Oct 2018, pp. 2245--2252.

\bibitem{YouTube_EventBasedMPC_QPVirtualConstraints}
{Quadrupedal Locomotion via Event-Based Predictive Control and QP-Based Virtual
  Constraints}, \url{https://youtu.be/qECrhD6ZMBA}.

\end{thebibliography}

\end{document}